\documentclass[journal]{IEEEtran}

\usepackage{cite}
\usepackage[T1]{fontenc} 
\usepackage{amsmath}
\usepackage[cmintegrals]{newtxmath}
\usepackage{bm} 
\usepackage{makecell}
\usepackage{setspace}
\usepackage{float}
\usepackage[pdftex]{graphicx}
\usepackage{color}
\usepackage{bbm}
\usepackage{multirow}
\usepackage{lipsum}
\usepackage{epstopdf}
\usepackage{algorithmic, algorithm}
\usepackage[absolute]{textpos}

\DeclareMathOperator*{\argmin}{arg\,min}
\ifCLASSINFOpdf
\else
\fi
\begin{document}

%
\title{$\ell_0$-Motivated Low-Rank Sparse Subspace Clustering}

\author{Maria Brbi\'c
        and~ Ivica Kopriva,~\IEEEmembership{Senior Member,~IEEE}

\thanks{This work was supported by the Croatian Science Foundation grant IP-2016-06-5235 (Structured Decompositions of Empirical Data for Computationally-Assisted Diagnoses of Disease) and by the European Regional Development Fund under the grant KK.01.1.1.01.0009 (DATACROSS).}
\thanks{M. Brbi\'c and I. Kopriva are with the Laboratory for Machine Learning and Knowledge Representation, Division of Electronics, Rudjer Boskovic Institute (e-mail: maria.brbic@irb.hr; ivica.kopriva@irb.hr).}
}

\markboth{ IEEE TRANSACTIONS ON CYBERNETICS}%
{Shell \MakeLowercase{\textit{et al.}}: Bare Demo of IEEEtran.cls for IEEE Journals}
%

\begin{textblock}{5}(11,0.5)
	\noindent { 10.1109/TCYB.2018.2883566}
\end{textblock}

\IEEEoverridecommandlockouts
\IEEEpubid{\makebox[\columnwidth]{978-1-5386-5541-2/18/\$31.00~\copyright2018 IEEE. \hfill} \hspace{\columnsep}\makebox[\columnwidth]{ }}
\maketitle

\IEEEpubidadjcol

\begin{abstract}
In many applications, high-dimensional data points can be well represented by low-dimensional subspaces. To identify the subspaces, it is important to capture a global and local structure of the data which is achieved by imposing low-rank and sparseness constraints on the data representation matrix. In low-rank sparse subspace clustering (LRSSC), nuclear and $\ell_1$ norms are used to measure rank and sparsity. However, the use of nuclear and $\ell_1$ norms leads to an overpenalized problem and only approximates the original problem. In this paper, we propose two $\ell_0$ quasi-norm based regularizations. First, the paper presents regularization based on multivariate generalization of minimax-concave penalty (GMC-LRSSC), which contains the global minimizers of $\ell_0$ quasi-norm regularized objective. Afterward, we introduce the Schatten-0 ($S_0$) and $\ell_0$ regularized objective and approximate the proximal map of the joint solution using a proximal average method ($S_0/\ell_0$-LRSSC). The resulting nonconvex optimization problems are solved using alternating direction method of multipliers with established convergence conditions of both algorithms. Results obtained on synthetic and four real-world datasets show the effectiveness of GMC-LRSSC and $S_0/\ell_0$-LRSSC when compared to state-of-the-art methods.
\end{abstract}


%
\IEEEpeerreviewmaketitle


\section{Introduction}

\IEEEPARstart{H}{igh} dimensional data analysis is a widespread problem in many applications of machine learning, computer vision, and bioinformatics \cite{Wang04, Wright09, Rao10, Morsier16, Song16, Liew11}. 
However, in many real-world datasets, the intrinsic dimension of high-dimensional data is much smaller than the dimension of the ambient space and data can be well represented as lying close to a union of low-dimensional subspaces. The problem of segmenting data according to the low-dimensional subspaces they are drawn from is known as subspace clustering \cite{Vidal11}. Thanks to their capability to handle arbitrarily shaped clusters and their well-defined mathematical principles, spectral based methods \cite{Shi2000, Ng01} are widely used approaches to subspace clustering. These methods solve the subspace clustering problem by relying on the spectral graph theory and cluster eigenvectors of the graph Laplacian matrix corresponding to the smallest eigenvalues \cite{vonLuxburg2007}.

One of the main challenges in subspace clustering is the construction of the affinity matrix that captures well (di)similarities between data points. Among various approaches proposed in the literature, methods based on sparse \cite{Elhamifar13, Peng17} and low-rank representations \cite{Liu10b, Liu13, WangWang17} have been among the most successful in many applications \cite{Li16}. These methods exploit the self-expressiveness property of the data and represent each data point as a linear combination of other data points in the dataset. Low-rank representation (LRR) \cite{Liu10b, Liu13, Favaro11} captures the global structure of the data by imposing a low-rank constraint on the data representation matrix. Low-rank implies that representation matrix is described by a weighted sum of small number of outer products of left and right singular vectors. In order to ensure convexity of the related optimization problem, the rank minimization is relaxed as the nuclear or Schatten-1 norm minimization problem \cite{Candes09,Candes10,Candes11}. Different from LRR, Sparse Subspace Clustering (SSC) \cite{Elhamifar09, Elhamifar13} captures local linear relationships by constraining representation matrix to be sparse. Using the tightest convex relaxation of the $\ell_0$ quasi-norm, the SSC model solves sparsity maximization problem as $\ell_1$ norm minimization problem \cite{Candes05, Donoho06}. Both LRR and SSC guarantee exact clustering when subspaces are independent, but the independence assumption is overly restrictive for many real-world datasets \cite{Tang14, Tang16}. Under appropriate conditions \cite{Elhamifar10}, SSC also succeeds for disjoint subspaces. However, when the number of dimensions is higher than three, SSC can face connectivity problems resulting in a disconnected graph within a subspace \cite{Nasihatkon11}. A natural way to construct adaptive model able to capture the global and the local structure of the data is to constrain representation matrix to be low-rank and sparse. In \cite{Zhuang12, Wang13, Li16, Brbic18} that is done by combining nuclear and $\ell_1$ norms as the measures of rank and sparsity, respectively. The motivation lies in the fact that minimization of these norms results in a convex optimization problem.

Although convex, nuclear and $\ell_1$ norms are not exact measures of rank and sparsity. Therefore, optimal solution of the nuclear and $\ell_1$ norms regularized objective is only approximate solution of the original problem \cite{Donoho03}. Proximity operator associated with the nuclear norm overpenalizes large singular values, leading to biased results in low-rank constrained optimization problems \cite{Larsson16, Parekh16}. Similarly, in sparsity regularized problems $\ell_1$ norm solution systematically underestimates high amplitude components of sparse vectors \cite{Selesnick17b}. Nonconvex regularizations based on $\ell_p$ quasi-norms ($0\leq p <1$) or their approximations have been proposed for various low-rank \cite{Larsson16, Parekh16, LuLin14, LuZhu15, Malek14, Kliesch16, LuTang16} and sparsity regularized problems \cite{Yuan17, Selesnick17b, Peharz12, Blumensath09, Mohimani09, Chartrand07, Daubechies04, XuChang12, Nikolova13}. 
Recently, nonconvex approximations of rank and sparsity have also been introduced in subspace clustering problem \cite{Yang16, Jiang16, Zhang16, Zhang18, Cheng17}. Specifically, $\ell_0$-induced sparse subspace clustering is introduced in \cite{Yang16}. The corresponding optimization problem is solved using proximal gradient descent which under assumptions on the sparse eigenvalues converges to a critical point. In \cite{Jiang16} authors replaced the nuclear norm regularizer with the nonconvex Ky Fan $p$-$k$ norm \cite{Doan16} and proposed proximal iteratively reweighted optimization algorithm to solve the problem. In \cite{Zhang16, Zhang19} rank is approximated using Schatten-$q$ quasi-norm regularization ($0<q<1$). The optimization problem in \cite{Zhang16} is solved using generalized matrix soft thresholding algorithm \cite{Zuo13}. Schatten-$q$ quasi-norm minimization with tractable $q=2/3$ and $q=1/2$ is proposed in \cite{Zhang18}. The Schatten-$q$ ($S_q$) quasi-norm for $0< q<1$ is equivalent to $\ell_q$ quasi-norm on vector of singular values. Compared to the nuclear norm, it makes a closer approximation of the rank function. In this regard, $S_0$ quasi-norm can be seen as an equivalent to the $\ell_0$ quasi-norm and stands for the definition of the rank function. Furthermore, \cite{Cheng17} combines $S_q$ regularizer ($0<q<1$) for low-rank and $\ell_p$ quasi-norm regularizer ($0<p<1$) for sparsity constraint. 
However, recent results in \cite{Zheng17} show that in $\ell_p$ regularized least squares ($0\leq p<1$) smaller values of $p$ lead to more accurate solutions. If $\ell_1$ norm is also considered, authors show that for large measurement noises $\ell_1$ outperforms $\ell_p$, $p<1$. However, for small measurement noises $\ell_0$ quasi-norm regularization outperforms $\ell_p$, $0<p\leq 1$.

Motivated by the limitations discussed above, we introduce two $S_0$/$\ell_0$ quasi-norm based nonconvex regularizations for low-rank sparse subspace clustering (LRSSC). 
First, we propose regularization based on multivariate generalization of the minimax-concave penalty function (GMC), introduced in \cite{Selesnick17b} for sparsity regularized linear least squares. Here, this approach is extended to the rank approximation. The GMC penalty enables to maintain the convexity of low-rank and sparsity constrained subproblems, while achieving better approximation of rank and sparsity than nuclear and $\ell_1$ norms. Importantly, this regularization is closely related to the continuous exact $\ell_0$ penalty which contains the global minimizers of $\ell_0$ quasi-norm regularized least-squares objective\cite{Soubies15, Selesnick17b}. GMC penalty yields solutions of low-rank and sparsity constrained subproblems based on firm thresholding of singular values and coefficients of representation matrix, respectively. The firm thresholding function  $\Theta: \mathbb{R}\rightarrow  \mathbb{R}$ is  is defined as \cite{Gao97}:
\begin{equation} \label{firm_thr}
\Theta(x; \lambda, a)=
\begin{cases}
0, & \text{if}\ |x|\leq\lambda \\
a(|x|-\lambda)/(a-\lambda)sign(x), & \text{if}\ \lambda\leq|x|\leq a \\
x, & \text{if}\ |x|\geq a.
\end{cases}
\end{equation}

Next, we propose the direct solution of $S_0$ and $\ell_0$ quasi-norms regularized objective function. The solution of corresponding low-rank and sparsity constrained subproblems is based on iterative application of hard thresholding operator \cite{Blumensath08, Le13, Liang16} on the singular values and coefficients of the representation matrix, respectively.  The hard thresholding function $H: \mathbb{R}\rightarrow  \mathbb{R}$ is defined as \cite{Blumensath08}:
\begin{equation} \label{hard_thr}
H(x; \lambda)=
\begin{cases}
x, & \text{if}\ |x|>\sqrt{2\lambda} \\
\{0,x\}, & \text{if}\ |x|=\sqrt{2\lambda} \\
0, & \text{if}\ |x|<\sqrt{2\lambda}.
\end{cases}
\end{equation}
Simultaneous rank and sparsity regularization is handled using the proximal average method, introduced in \cite{YuYao13} for convex problems and extended recently to nonconvex and nonsmooth functions \cite{YuXun15, LinWei16}. Proximal average allows us to approximate the proximal map of joint solution by averaging solutions obtained separately from low-rank and sparsity subproblems, leading to a problem with a low computational cost in each iteration. Furthermore, using proximal average method enables us to establish global convergence guarantee for $S_0/\ell_0$ regularized LRSSC. 

Better approximation of rank and sparsity is a consequence of the properties of firm and hard thresholding operators associated with GMC and $\ell_0$ regularizations. As opposed to them, the soft thresholding operator underestimates high amplitude coefficients in $\ell_1$ norm based sparsity regularized objective, as well as large singular values in low-rank approximation problem. As an example, Fig. \ref{Fig1} shows soft, firm and hard thresholding operators used in LRSSC, GMC-LRSSC and $S_0/\ell_0$-LRSSC, respectively. 

\begin{figure}[h]
	\centering
	\includegraphics[width=0.48\textwidth]{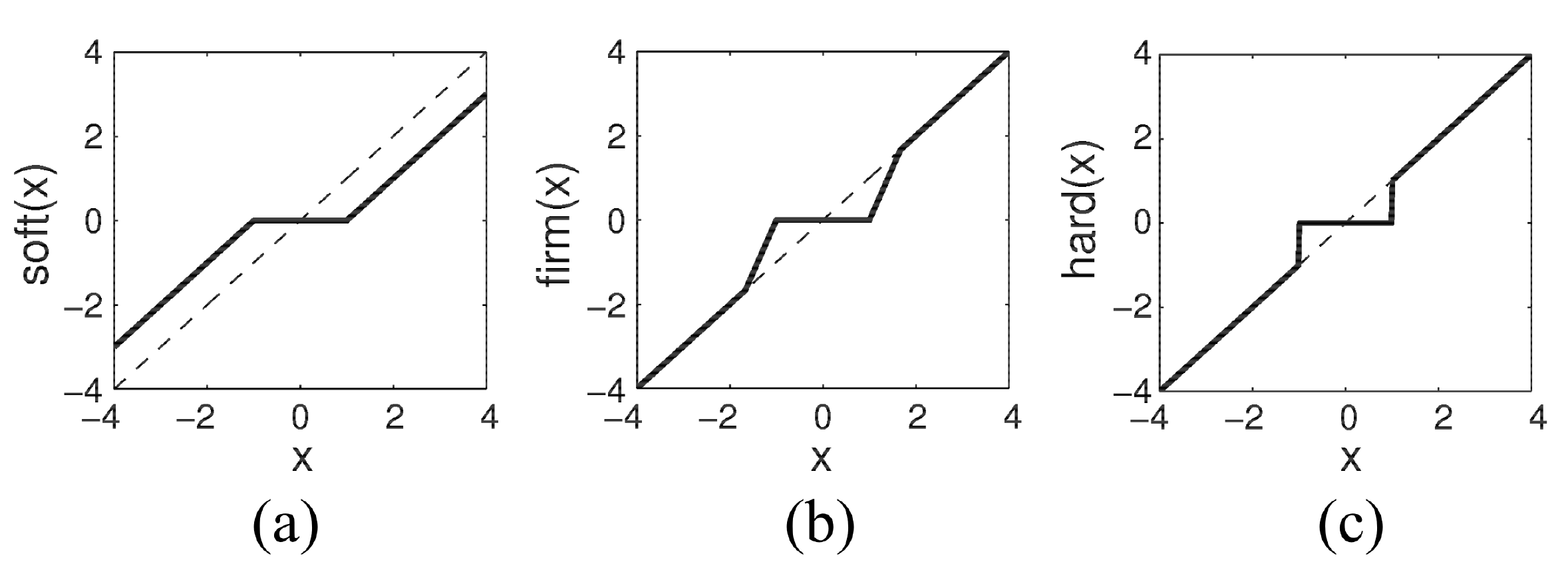}
	\caption{Proximity operators for threshold value $\lambda=1$. (a) Soft-thresholding operator $soft(x;\lambda)=sign(x)max(0,|x|-\lambda)$ is associated with $\ell_1$ norm. (b) Firm-thresholding operator defined in (\ref{firm_thr}) and associated with the scaled MC penalty and used in GMC-LRSSC formulation. Parameter $a$ is for visualization proposes set to $0.6$. (c) Hard-thresholding operator defined in (\ref{hard_thr}) and associated with $\ell_0$ quasi-norm.} 
	\label{Fig1}
\end{figure}

To solve corresponding optimization problems we derive algorithms based on computationally efficient Alternating Direction Method of Multipliers (ADMM) \cite{Boyd11}. Although ADMM has been successfully applied for many nonconvex problems \cite{Sun14, Zhang14, Dong13}, only recent theoretical results establish convergence of ADMM for certain nonconvex functions \cite{Hong16, WangYin15, LiTing15, WangCao15}. For GMC regularization, we show that the sequence generated by the algorithm is bounded and prove that any limit point of the iteration sequence is a stationary point. For $S_0/\ell_0$ regularization with proximal average approach, based on the property that $\ell_0$ and $S_0$ quasi-norms belong to a class of semialgebraic functions and satisfy the Kurdyka-\L{}ojasiewicz inequality \cite{Attouch13}, the global convergence of the algorithm can be guaranteed.
 
Experimental results on synthetic and four real-world datasets demonstrate that the proposed $\ell_0$ based low-rank sparse subspace clustering algorithms converge fast and to a point with lower or similar clustering error than the convex approximations with nuclear and $\ell_1$ norms. Compared to the state-of-the-art subspace clustering methods, the proposed algorithms perform better on four benchmark datasets.

\subsection{Contributions}
The contributions of this paper are summarized as follows:
\begin{enumerate}
\item We introduce nonconvex generalized minimax-concave penalty in the low-rank sparse subspace clustering problem, such that the global minimizers of the proposed objective coincide with that of a convex function defined using the continuous exact $\ell_0$ penalty \cite{Soubies15}. The introduced penalty maintains the convexity of the sparsity and low-rank constrained subproblems. The proximal operator of the related GMC penalty function is the firm thresholding function \cite{Selesnick17b}. 
\item We introduce $S_0$ and $\ell_0$ pseudo-norm regularizations for LRSSC. Using the proximal average method \cite{YuXun15, YuYao13}, we average the solutions of proximal maps of low-rank and sparsity subproblems, with the hard thresholding function as a proximity operator of the related penalties \cite{Blumensath08}.
\item We derive ADMM based optimization algorithms for LRSSC constrained either with a GMC penalty or with $S_0$/$\ell_0$ quasi-norms. Iterative firm or hard thresholding of singular values and coefficients of representation matrix is used to obtain the solution of rank and sparsity constrained subproblems.
\item We prove that the sequence generated by the GMC regularized LRSSC algorithm is bounded and that any limit point of the iteration sequence is a stationary point that satisfies Karush-Kuhn-Tucker (KKT) conditions. 
\item We establish the convergence property of the $S_0$/$\ell_0$ regularized approach with proximal average and show that the algorithm converges regardless of the initialization. To the best of our knowledge, we are the first to show convergence with $S_0$ and $\ell_0$ penalties in the low-rank and sparsity constrained optimization problem.
\end{enumerate}

The remainder of this paper is organized as follows. Section II gives a brief overview of the related work. In Section III and IV we introduce GMC and $S_0/\ell_0$ regularized low-rank sparse subspace clustering methods, respectively. We formulate the problem, present optimization algorithms, and analyze convergence and computational complexity. The experimental results on synthetic and four real-world datasets are presented in Section V. Finally, Section VI concludes this paper.

\subsection {Main Notation}

Scalars are denoted by lower case letters, vectors by bold lower-case letters, matrices are denoted by bold capital and subspaces by calligraphic letters. $\|\cdot\|_F$ denotes Frobenius norm defined as the square root of the sum of the squares of matrix elements. $\|\cdot\|_1$ denotes $\ell_1$ norm defined as the sum of absolute values of matrix elements. $\|\cdot\|_*$ denotes nuclear norm defined as the sum of singular values of a matrix. $\ell_0$ quasi-norm is denoted by $\|\cdot\|_0$ and 
for matrix $\bm{A} \in \mathbb{R}^{N\times M}$ defined as: 
\begin{equation*}
\|\bm{A}\|_0 = \# \big\{ a_{ij} \neq 0, \  i=1..N, \ j=1..M\big\},
\end{equation*}
where $\#$ denotes cardinality function. Schatten-$0$ quasi norm is denoted by $\|\cdot\|_{S_0}$ and defined as:
\begin{equation*}
\|\bm{A}\|_{S_0}= \|diag(\bm \Sigma)\|_0,
\end{equation*}
where $\bm{A}=\bm{U\Sigma V}^T$ is the singular value decomposition (SVD) of matrix $\bm{A}$. 
Since $\ell_0$ quasi-norm does not satisfy homogeneous property it is not a norm, but with a slight abuse of notation we will refer to it as the $\ell_0$ norm in the rest of the paper.
Null vector is denoted by $\bm{0}$ and $diag(\cdot)$ is the vector of diagonal elements of a matrix. Table \ref{notation_summary} summarizes some notations used in the paper.

\begin{table}[h]
\caption{Notations and abbreviations}
\fontsize{9}{11}\selectfont
\begin{center} \label{notation_summary} 
\begin{tabular}{ l  l  } 
\hline 
 \textbf{Notation} & \textbf{Definition} \\ 
 \hline
$N$ & Number of data points  \\ 
$n$ & Dimension of data points  \\ 
$L$ & Number of subspaces  \\ 
$\bm X \in \mathbb{R}^{n\times N}$ & Data matrix \\ 
$\bm C\in \mathbb{R}^{N\times N}$ & Representation matrix \\   
$\bm W\in \mathbb{R}^{N\times N}$ & Affinity matrix   \\ 
$\bm X=\bm{U\Sigma} \bm{V}^T$ & Singular value decomposition of $\bm X$   \\
$\sigma(\bm{X})$ & Vector of singular values of $\bm{X}$ \\
 \hline
\end{tabular}
\end{center}
\end{table} 

\section{Background}

Consider the data matrix $\bm X \in \mathbb{R}^{n\times N}$ the columns of which are data points drawn from a union of $L$ linear subspaces $\bigcup_{i=1}^{L}\mathcal{S}_i$ of unknown dimensions $\big\{d_i=\textnormal{dim}(\mathcal{S}_i)\big\}_{i=1}^L$ in $\mathbb{R}^n$. Let $\bm X_i \in \mathbb{R}^{n\times N_i}$ be a submatrix of $\bm X$ of rank $d_i$, $0<d_i<N_i$ and $\sum_{i=1}^{L}N_i=N$. Given data matrix $\bm X$, subspace clustering segments data points according to the low-dimensional subspaces. The first step is the construction of the affinity matrix $\bm W \in \mathbb{R}^{N\times N}$ whose elements represent the similarity between data points. An ideal affinity matrix is block diagonal (up to a permutation): non-zero distance is assigned to the points in the same subspace and zero distance to the points from different subspaces. Spectral clustering algorithm \cite{Shi2000, Ng01} is then applied to the affinity matrix to obtain memberships of data points to the subspaces. 


\subsection{Related Work}

Low-Rank Representation (LRR) \cite{Liu10b, Liu13} aims to find a low-rank representation matrix $\bm C\in \mathbb{R}^{N\times N}$ for input data matrix $\bm X$ by solving the following convex optimization problem:
\begin{equation} \label{LRR}
\min_{\bm C}\big\|\bm C\big\|_*\ \  s.t.\ \ \bm X=\bm{XC},
\end{equation}
where the nuclear norm is used to approximate the rank of $\bm C$.
Let $\bm X=\bm{U\Sigma V}^T$ be the SVD of $\bm X$. The closed form solution of problem (\ref{LRR}) is given by \cite{Liu13}:
\begin{equation}
\hat{\bm C}=\bm{V}\bm{V}^T.
\end{equation}

When data points are contaminated by additive white Gaussian noise (AWGN), the following minimization problem is solved:
\begin{equation} \label{LRR_noise}
\min_{\bm C}\frac{\lambda}{2}\big\|\bm X-\bm{XC}\big\|_F^2+\big\|\bm C\big\|_*,
\end{equation}
where $\lambda$ is the rank regularization constant.
The optimal solution of problem (\ref{LRR_noise}) is given by \cite{Favaro11, Vidal14}:
\begin{equation}
\hat{\bm C}=\bm{V}_1(\bm{I}-\frac{1}{\lambda}\bm \Sigma_1^{-2})\bm{V_1}^T,
\end{equation}
where $\bm{U}=[\bm{U}_1\ \bm{U}_2]$, $\bm{\bm \Sigma}=diag(\bm{\bm \Sigma}_1\ \bm{\bm \Sigma}_2)$ and $\bm{V}=[\bm{V}_1\ \bm{V}_2]$. Matrices are partitioned according to the sets $\mathcal{I}_1=\{i:\sigma_i>\frac{1}{\sqrt\lambda}\}$ and $\mathcal{I}_2=\{i:\sigma_i\leq\frac{1}{\sqrt\lambda}\}$, where $\sigma_i$ denotes $i$th singular value of $\bm X$.

Sparse Subspace Clustering (SSC) \cite{Elhamifar13} represents each data point as a sparse linear combination of other data points and solves the following convex optimization problem:
\begin{equation}
\min_{\bm C}\big\|\bm C\big\|_1\ \  s.t.\ \ \bm X=\bm{XC},\ diag(\bm C)=\bm{0},
\end{equation} 
where constraint $diag(\bm C)=\bm{0}$ is used to avoid trivial solution of representing a data point as a linear combination of itself.

For data contaminated by the AWGN, the following minimization problem is solved to approximate sparse representation matrix $\bm C$:
\begin{equation}
\min_{\bm C}\frac{1}{2}\big\|\bm X-\bm{XC}\big\|_F^2+\tau\big\|\bm C\big\|_1\ \  s.t.\ \ diag(\bm C)=\bm{0},
\end{equation}
where $\tau$ is the sparsity regularization constant.
This problem can be solved efficiently using ADMM optimization procedure \cite{Boyd11, Elhamifar13}.

Low-Rank Sparse Subspace Clustering (LRSSC) \cite{Wang13} requires that the representation matrix $\bm C$ is simultaneously low-rank and sparse. LRSSC solves the following problem:
\begin{equation}
\min_{\bm C}\lambda\big\|\bm C\big\|_*+\tau\big\|\bm C\big\|_1\ \  s.t.\ \  \bm X=\bm{XC},\ diag(\bm C)=\bm{0},
\end{equation}
where $\lambda$ and $\tau$ are rank and sparsity regularization constants, respectively.
For the AWGN corrupted data the following problem needs to be solved to approximate $\bm C$:
\begin{equation} \label{LRSSC_obj}
\begin{split}
\min_{\bm C}\frac{1}{2}\big\|\bm X&-\bm{XC}\big\|_F^2+\lambda\big\|\bm C\big\|_*+\tau\big\|\bm C\big\|_1\\&\ \ s.t.\ \  diag(\bm C)=\bm{0}.
\end{split}
\end{equation}

After representation matrix $\bm C$ is estimated, the affinity matrix $\bm W\in \mathbb{R}^{N\times N}$ is calculated as follows:
\begin{equation}
\bm W=|\bm C|+|\bm C|^T.
\end{equation}

In the next two sections, we introduce two nonconvex regularizers for the low-rank sparse subspace clustering. We formulate low-rank sparse subspace clustering problem in the following general form:
\begin{equation} \label{LRSSC_gen}
\begin{split}
\min_{\bm C}\frac{1}{2}\big\|\bm X&-\bm{XC}\big\|_F^2+\lambda g(\bm C)+\tau f(\bm C)\\&\ \  s.t.\ \ diag(\bm C)=\bm{0},
\end{split}
\end{equation}
where $g(\bm C)$ and $f(\bm C)$ are functions that, respectively, measure rank and sparsity of the data representation matrix $\bm C$. The convex formulation used in (\ref{LRSSC_obj}) implies $g(\bm C)=\big\|\bm C\big\|_*$ and $f(\bm C)=\big\|\bm C\big\|_1$.

\section{GMC-LRSSC Algorithm}

\subsection{Problem Formulation}
We propose to regularize rank and sparsity using multivariate GMC penalty function, introduced in \cite{Selesnick17b} for sparse regularized least-squares. 
We start with some definitions and results that will be used throughout the paper.

\newtheorem{definition}{Definition}
\newtheorem{lemma}{Lemma}
\newtheorem{proposition}{Proposition}
\newtheorem{theorem}{Theorem}
\newenvironment{proof}
{\textit{Proof:} }
{$\square$}

\begin{definition}[\cite{Selesnick17b}]
	Let $\bm z \in \mathbb{R}^{N}$ and $\bm B \in \mathbb{R}^{M\times N}$. The GMC penalty function $\psi_{\bm B}: \mathbb{R}^{N}\rightarrow  \mathbb{R}$ is defined as:
	\begin{equation} \label{GMC_def}
	\psi_{\bm B}(\bm z)=\|\bm z\|_1-S_{\bm B}(\bm z),
	\end{equation}
	where $S_{\bm B}:\mathbb{R}^{N} \rightarrow \mathbb{R}$ is the generalized Huber function defined as:
	\begin{equation} \label{Huber_def}
		S_{\bm B}(\bm z)=\inf_{\textbf{v}\in \mathbb{R}^{N}}\Big\{\|\bm{v}\|_1+\frac{1}{2}\|\bm B(\bm z-\bm{v})\|_2^2\Big\}.
		\end{equation}
\end{definition}

\begin{lemma}[\cite{Selesnick17b}]\label{lemma2}
	Let $\bm z\in \mathbb{R}^N$, $\bm{y}\in \mathbb{R}^M$, $\bm{A} \in \mathbb{R}^{M\times N}$ and $\lambda>0$. Define $F:\mathbb{R}^N\rightarrow \mathbb{R}$ as:
	\begin{equation}  \label{sparsity_reg}
	F(\bm z)=\frac{1}{2}\|\bm{y}-\bm{Az}\|_2^2+\lambda\psi_{\bm B}(\bm z),
	\end{equation}
	where $\psi_{\bm B}:\mathbb{R}^N\rightarrow \mathbb{R}$ is the GMC penalty. If $\bm{A}^T\bm{A}-\lambda\bm B^T\bm B$ is positive semidefinite matrix, $F$ is a convex function. The convexity condition is satisfied by setting:
	\begin{equation} \label{B_cond}
	\bm B = \sqrt{\gamma/\lambda}\bm{A}, \quad 0\leq\gamma\leq 1.
\end{equation}\end{lemma}

The parameter $\gamma$ controls the nonconvexity of the penalty $\psi_{\bm B}$. Larger values of $\gamma$ increase the nonconvexity of the penalty. $\ell_1$ norm can be seen as a special case of this penalty by setting $\gamma=0$.

\begin{lemma}[\cite{Selesnick17b}]\label{lemma3}
	Let $\bm z\in \mathbb{R}^N$, $\bm{y}\in \mathbb{R}^M$, $\bm{A} \in \mathbb{R}^{M\times N}$ and $\lambda>0$. If $\bm{A}^T\bm{A}$ is diagonal with positive entries and $\bm B$ is given by (\ref{B_cond}), then for $0< \gamma \leq 1$ the minimizer of $F$ is given by element-wise firm thresholding. Formally, if
	\begin{equation}
	\bm{A}^T\bm{A} = diag(\alpha_1^2,...,\alpha_N^2),
	\end{equation}
	then
	\begin{equation}
	\bm z_n^{\text{opt}}=\Theta([\bm{A}^T\bm{y}]_n/\alpha_n^2;\lambda/\alpha_n^2,\lambda/(\gamma\alpha_n^2)),
	\end{equation}
	where $\Theta$ stands for the firm thresholding function \cite{Gao97} defined entry-wise in (\ref{firm_thr}).
\end{lemma}

\begin{definition}[\cite{Lewis05, Sun17}]\label{def3}
	Function $f: \mathbb{R}^N\rightarrow \mathbb{R}$ is an absolutely symmetric function, if:
	\begin{equation}
	f(z_1,z_2,...,z_N)=f(|z_{\pi(1)}|,|z_{\pi(2)}|,...,|z_{\pi(N)}|),
	\end{equation} 
	holds for any permutation $\pi$ of $\{1,...,N\}$.
\end{definition}

\begin{proposition} \label{propos1}
	Let $\bm B^T\bm B$ be a diagonal matrix and $\psi_{\bm B}$ be the GMC penalty function defined in (\ref{GMC_def}). The subdifferential of singular value function $\psi_{\bm B} \circ \sigma$ of a matrix $\bm X$ is given by the following equation:
	\begin{equation}
	\partial [(\psi_{\bm B} \circ \sigma)(\bm X)]=\bm{U}diag(\partial\psi_{\bm B}[\sigma_i(\bm X)])\bm{V}^T,
	\end{equation}
	where $\bm X=\bm{U}\bm \Sigma\bm{V}^T$ is the SVD of $\bm X$.
\end{proposition}

\begin{proof}
It follows from \cite{Selesnick17b} that if $\bm B^T\bm B$ is a diagonal matrix, the GMC penalty $\psi_{\bm B}$ is separable, comprising a sum of scalar MC penalties:
	\begin{equation}
	\bm B^T\bm B = diag({\alpha_1}^2,...,{\alpha_N}^2) \Rightarrow \psi_{\bm B}(\bm z)=\sum_{n=1}^{N}\phi_{\alpha_n}(z_n),
	\end{equation}
	where $\phi_b:\mathbb{R}\rightarrow\mathbb{R}$ is the scaled MC penalty \cite{Zhang10, Selesnick17b} defined as:
		\begin{equation} {\label{scaled_MC_penalty}}
		\phi_b(y) = 
		   \begin{cases}
		    |y|-\frac{1}{2}b^2y^2, & \text{if}\ |y|\leq1/b^2, \\
		      \frac{1}{2b^2}, & \text{otherwise}.
		    \end{cases}
		\end{equation}
Therefore, according to Definition \ref{def3}, $\psi_{\bm B}$ is an absolutely symmetric function. The proof of the proposition then follows from the property of the singular value function $f\circ \sigma$ \cite{Sun17}, where $f$ is an absolutely symmetric function. 
\end{proof}

Proposition \ref{propos1} allows us to use GMC penalty for rank approximation. We formulate GMC penalty regularized objective for low-rank sparse subspace clustering. Let $\bm B\in\mathbb{R}^{N\times N}$, and let ${\sigma}(\bm C)$ denote vector of singular values of $\bm C$. By choosing $g(\bm C)=\psi_{\bm B}({\sigma}(\bm C))$ as a rank function, and $f(\bm C)=\psi_{\bm B}(\bm C)$ as a sparsity function in equation (\ref{LRSSC_gen}), we define the following nonconvex objective function:
\begin{equation} \label{LRSSC_gmc}
\begin{split}
\min_{\bm C}\frac{1}{2}\big\|\bm X-&\bm{XC}\big\|^2_F+\lambda \psi_{\bm B}({\sigma}(\bm C))+\tau \psi_{\bm B}(\bm C)\\&\ \  s.t.\ \ diag(\bm C)=\bm{0},
\end{split}
\end{equation}
where $\psi_{\bm B}$ denotes GMC penalty defined in (\ref{GMC_def}), regularized by matrix $\bm B$. In the next section we will show that by solving the objective (\ref{LRSSC_gmc}) with ADMM, both sparsity and low-rank subproblems can be reduced to the equation (\ref{sparsity_reg}) with diagonal $\bm{A}^T\bm{A}$. In this case, GMC penalty is closely related to the continuous exact $\ell_0$ penalty \cite{Selesnick17b, Soubies15}, that approximates the convex hull of the least squares with $\ell_0$ regularization. Furthermore, diagonal $\bm{A}^T\bm{A}$ reduces both subproblems to element-wise firm thresholding function, defined in (\ref{firm_thr}). In low-rank minimization subproblem, the firm thresholding operator needs to be applied to the vector of singular values. 

\subsection{Optimization Algorithm}

To solve optimization problem in (\ref{LRSSC_gmc}), we introduce auxiliary variables $\bm J$, $\bm C_1$ and $\bm C_2$ to split variables and solve subproblems independently. The reformulated objective for GMC penalty in (\ref{LRSSC_gmc}) is equivalent to:
\begin{equation} \label{LRSSC_gmc_reform}
\begin{split}
\min_{\bm J, \bm C_1, \bm C_2}& \frac{1}{2}\big\|\bm X-\bm{XJ}\big\|^2_F+\lambda \psi_{\bm B}(\sigma(\bm C_1))+\tau \psi_{\bm B}(\bm C_2)\ \  \\& s.t.\ \ \bm J=\bm C_1,\ \ \bm J=\bm C_2-diag(\bm C_2),
\end{split}
\end{equation}
The augmented Lagrangian function of (\ref{LRSSC_gmc_reform}) is:
\begin{equation} \label{Lag}
\begin{split}
&\mathcal{L}_{\mu_1,\mu_2}\big(\bm J,\bm C_1,\bm C_2, \bm\Lambda_1,\bm\Lambda_2\big) = \frac{1}{2}\big\|\bm X-\bm{XJ}\big\|^2_F+\lambda \psi_{\bm B}( \sigma(\bm C_1))\\&+\tau \psi_{\bm B}(\bm C_2)+\frac{\mu_1}{2}\big\|\bm J-\bm C_1\big\|^2_F+\frac{\mu_2}{2}\big\|\bm J-\bm C_2+diag(\bm C_2)\big\|^2_F \\&+\big\langle\bm\Lambda_1, \bm J-\bm C_1\big\rangle+\big\langle\bm\Lambda_2, \bm J-\bm C_2+diag(\bm C_2)\big\rangle,
\end{split}
\end{equation}
where $\mu_1$, $\mu_2>0$ are penalty parameters and $\bm\Lambda_1$, $\bm\Lambda_2$ are Lagrange multipliers.

\textit{Update rule for $\bm J^{k+1}$}: Given $\bm C_1^k$, $\bm C_2^k$, $\bm\Lambda_1^k$, $\bm\Lambda_2^k$, $\mu_1^k$, $\mu_2^k$, we minimize the Lagrangian function in (\ref{Lag}) with respect to $\bm J$:
\begin{equation} \label{J_update}
\begin{split}
&\min_{\bm J}\mathcal{L}_{\mu_1^k,\mu_2^k}\big(\bm C_1^k,\bm C_2^k, \bm J, \bm\Lambda_1^k, \bm\Lambda_2^k\big)=
\min_{\bm J}\frac{1}{2}\big\|\bm X-\bm{XJ}\big\|_F^2\\&+\frac{\mu_1}{2}\big\|\bm J-\bm C_1^k\big\|_F^2+\frac{\mu_2}{2}\big\|\bm J-\bm C_2^k+diag(\bm C_2^k)\big\|^2_F+\big\langle\bm\Lambda_1^k, \bm J-\bm C_1^k\big\rangle \quad\\&+\big\langle\bm\Lambda_2^k,\bm J-\bm C_2^k+diag(\bm C_2^k)\big\rangle.
\end{split}
\end{equation}
The optimal solution of (\ref{J_update}) is given by the following update:
\begin{equation} \label{J_rule}
\bm J^{k+1}= \big[\bm X^T\bm X+(\mu_1^k+\mu_2^k)\bm{I}\big]^{-1}\big[\bm X^T\bm X+\mu_1^k{\bm C_1}^k+\mu_2^k{\bm C_2}^k-{\bm\Lambda_1}^k-{\bm\Lambda_2}^k\big].
\end{equation}
\textit{Update rule for $\bm C_1^{k+1}$:} Given $\bm J^{k+1}$, $\bm\Lambda_1^k$, $\mu_1^k$, we minimize the Lagrangian function in (\ref{Lag}) with respect to $\bm C_1$:
\begin{equation}\label{C1_min_gmc}
\begin{split}
&\min_{\bm C_1}\mathcal{L}_{\mu_1^k,\mu_2^k}\big(\bm J^{k+1}, \bm C_1,\bm C_2^k, \bm\Lambda_1^k, \bm\Lambda_2^k\big)\\&=\min_{\bm C_1}\lambda \psi_{\bm B}(\sigma(\bm C_1))+\frac{\mu_1^k}{2}\big\|\bm J^{k+1}-\bm C_1\big\|_F^2+\big\langle\bm\Lambda_1^k, \bm J^{k+1}-\bm C_1\big\rangle\\
&=\min_{\bm C_1}\lambda \psi_{\bm B}(\sigma(\bm C_1)) +\frac{\mu_1^k}{2}\Big\|\bm J^{k+1}+\frac{\bm\Lambda_1^k}{\mu_1^k}-\bm C_1\Big\|_F^2.
\end{split}
\end{equation}

It can be seen that (\ref{C1_min_gmc}) corresponds to the least squares problem in (\ref{sparsity_reg}) with $\bm{A}^T\bm{A}=\bm{I}$ and therefore, diagonal. It follows from the condition ($\ref{B_cond}$) that in order to maintain convexity of the subproblem, we need to set $\bm B=\sqrt{\mu_1^k\gamma/\lambda}\bm{I}$, $0< \gamma \leq 1$. Using Lemma \ref{lemma3} and Proposition \ref{propos1}, (\ref{C1_min_gmc}) can be solved by element-wise firm thresholding of singular values of matrix $\big(\bm J^{k+1}+\bm\Lambda_1^k/\mu_1^k\big)$.

Specifically, let $\bm{U\Sigma V}^T$ denote the SVD of matrix $\big(\bm J^{k+1}+{\bm\Lambda_1}^k/\mu_1^k\big)$. The closed-form solution of (\ref{C1_min_gmc}) is given by:
\begin{equation} \label{C1_update_gmc}
{\bm C_1}^{k+1}=\bm{U}\Theta\Big(\bm \Sigma;\frac{\lambda}{\mu_1^k},\frac{\lambda}{\gamma\mu_1^k}\Big)\bm{V}^T,
\end{equation}
where $\Theta$ is the firm thresholding function defined in (\ref{firm_thr}).

\textit{Update rule for $\bm C_2^{k+1}$:} Given $\bm J^{k+1}$, $\bm\Lambda_2^k$, $\mu_2^k$, we minimize the objective (\ref{Lag}) with respect to $\bm C_2$:
\begin{equation}\label{C2_min_gmc}
\begin{split}
&\min_{\bm C_2}\mathcal{L}_{\mu_1^k,\mu_2^k}\big(\bm J^{k+1}, \bm C_1^{k+1},\bm C_2, \bm\Lambda_1^k, \bm\Lambda_2^k\big)\\
&=\min_{\bm C_2}\tau\psi_{\bm B}(\bm C_2)+\frac{\mu_2^k}{2}\Big\|\bm J^{k+1}+\frac{\bm\Lambda_2^k}{\mu_2^k}-\bm C_2\Big\|_F^2,
\end{split}
\end{equation}
with subtraction of diagonal elements of $\bm C_2^{k+1}$:
\begin{equation} \label{diag_sub}
\bm C_2^{k+1}\leftarrow{\bm C_2}^{k+1}-diag\big({\bm C_2}^{k+1}\big).
\end{equation} 
Similarly as the update for $\bm C_1$, matrix $\bm{A}^T\bm{A}$ is diagonal matrix and we can ensure convexity of the subproblem (\ref{C2_min_gmc}) by setting $\bm B=\sqrt{\mu_2^k\gamma/\tau}\bm{I}$, $0< \gamma \leq 1$. The problem (\ref{C2_min_gmc}) is then solved by firm thresholding elements of matrix $\big(\bm J^{k+1}+{\bm\Lambda_2^k}/\mu_2^k\big)$ and given by:
\begin{equation} \label{C2_update_gmc}
\begin{split}
&{\bm C_2}^{k+1}=\Theta\Big(\bm J^{k+1}+\frac{\bm\Lambda_2^k}{\mu_2^k};\frac{\tau}{\mu_2^k},\frac{\tau}{\gamma\mu_2^k}\Big),
\\&\bm C_2^{k+1}\leftarrow{\bm C_2}^{k+1}-diag\big({\bm C_2}^{k+1}\big).
\end{split}
\end{equation}
\textit{Update rules for Lagrange multipliers $\bm\Lambda_1^{t+1},\bm\Lambda_2^{t+1}$:} Given $ \bm J^{k+1}$, $\bm C_1^{k+1}$, $\bm C_2^{k+1}$, $\mu_1^k$, $\mu_2^k$, Lagrange multipliers are updated with the following equations:
\begin{equation}\label{Lambda_rule}
\begin{split}
{\bm\Lambda_1^{k+1}}&={\bm\Lambda_1^k}+\mu_1^k\big(\bm J^{k+1}-\bm C_1^{k+1}\big)\\
{\bm\Lambda_2^{k+1}}&={\bm\Lambda_2^k}+\mu_2^k\big(\bm J^{k+1}-\bm C_2^{k+1}\big).\\
\end{split}
\end{equation}
Penalty parameters $\mu_1$, $\mu_2$ are are in each step $k$ updated according to:
\begin{equation} \label{mu_update}
{\mu_i}^{k+1}=\min(\rho{\mu_i}^{k}, {\mu}^{max}),\ \ i=1,2, 
\end{equation}
where $\rho>1$ is step size for adaptively changing $\mu_1, \mu_2$. Due to numerical reasons $\mu_1$, $\mu_2$ are bounded with $\mu^{max}$, while in the convergence proof we use formulation ${\mu_i}^{k+1}=\rho{\mu_i}^{k}$, $i=1,2$.

The main steps of the proposed algorithm are summarized in Algorithm \ref{alg1}.

\renewcommand{\algorithmicrequire}{\textbf{Input:}}
\renewcommand{\algorithmicensure}{\textbf{Output:}}
\begin{algorithm}[H]
	\caption{GMC-LRSSC by ADMM optimization}
	\label{alg1}
	\begin{algorithmic}[1]
		\REQUIRE Data points as columns in $\bm X $ 
		, $\{\tau, \lambda\}>0$, $0<\gamma\leq 1$
		\ENSURE Assignment of the data points to $k$ clusters\\
		\STATE Initialize: $\big\{\bm J, \bm C_1,\bm C_2,\bm\Lambda_1, \bm\Lambda_2\big\}=\bm 0$, $\{\mu_i^{(0)}>0\}_{i=1}^2$, $\rho>1$\\
		\STATE Compute $\bm X^T\bm X$ for later use\\
		\WHILE{not converged}
		\STATE Update $\bm J^{k+1}$ by (\ref{J_rule})\\
		\STATE Normalize columns of $\bm J$ to unit $\ell_2$ norm\\
		\STATE Update $\bm C_1^{k+1}$ by (\ref{C1_update_gmc}) \\
		\STATE Update $\bm C_2^{k+1}$ by (\ref{C2_update_gmc})\\
		\STATE Update $\bm\Lambda_1^{k+1}$, $\bm\Lambda_2^{k+1}$ by (\ref{Lambda_rule})\\ 
		\STATE Update $\mu_i^{k+1}=\min(\rho\mu_i^{k}, \mu^{max})$, $i=1,2$\\
		\ENDWHILE
		\STATE Calculate affinity matrix $\bm W=|\bm C_{1}|+|\bm C_{1}|^T$\\
		\STATE Apply spectral clustering \cite{Ng01} to $\bm W$\\
	\end{algorithmic}
\end{algorithm}

\subsection{Convergence Analysis}

Although choosing $\gamma \in [0,1]$ guarantees convexity of the low-rank and sparsity subproblems and convergence of related subproblems, the objective in (\ref{LRSSC_gmc}) is nonconvex. In this section, we analyze convergence of the proposed method and show that any limit point of iteration sequence satisfies Karush-Kuhn-Tucker (KKT) conditions \cite{Kuhn51}.

\begin{proposition} \label{boundness}
	The sequences $\big\{(\bm J^k, \bm C_1^k, \bm C_2^k, \bm \Lambda_1^k, \bm\Lambda_2^k)\big\}$ generated by Algorithm 1 are all bounded.
\end{proposition}

We now state the main theorem related to convergence property of GMC-LRSSC algorithm.

\begin{theorem} \label{thm1}
	Let $Y^k= \big\{\big(\bm J^k, \bm C_1^k, \bm C_2^k, \bm \Lambda_1^k, \bm \Lambda_2^k\big)\big\}_{k=1}^\infty$ be a sequence generated by Algorithm 1. Suppose that $\lim_{k\rightarrow\infty}\big(Y^{k+1}-Y^{k}\big)=\bm 0$. Then, any accumulation point of the sequence $\{Y^k\big\}_{k=1}^\infty$ satisfies the Karush-Kuhn-Tucker (KKT) conditions for problem (\ref{LRSSC_gmc_reform}). In particular, whenever $\{Y^k\big\}_{k=1}^\infty$ converges, it converges to a point that satisfies KKT conditions. 
\end{theorem}

The proofs of Proposition \ref{boundness} and Theorem \ref{thm1} are given in the Appendix.

\subsection{Stopping Criteria and Computational Complexity}

The steps in Algorithm 1 are repeated until convergence or until the maximum number of iterations is exceeded. We check the convergence by verifying the following inequalities at each iteration $k$: $\big\|\bm J^k - \bm C_1^k\|_{\infty}\leq\epsilon$, $\big\|\bm J^k - \bm C_2^k\|_{\infty}\leq\epsilon$, $\big\|\bm J^k - \bm J^{k-1}\|_{\infty}\leq\epsilon$. We found that setting error tolerance to $\epsilon=10^{-4}$ works well in practice. In each step we normalize columns of matrix $\mathbf{J}$. This normalization is frequently applied to stabilize convergence of non-negative matrix factorization algorithms \cite{Cichocki06}.

The computational complexity of Algorithm 1 is $O(nN^2+TN^3)$, where $T$ denotes the number of iterations. In the experiments, we set the maximal $T$ to $100$, but on all datasets the algorithm converged within less than 15 iterations. Note that the computational complexity of spectral clustering step is $O(N^3)$.

\section{$S_0/\ell_0$-LRSSC Algorithm}

\subsection{Problem Formulation}
In addition to the GMC penalty, we propose to directly use $S_0$ and $\ell_0$ as constrains for low-rank and sparsity. Specifically, by choosing $g(\bm C)=\|\bm C\|_{S_0}$ as a rank function, and $f(\bm C)=\|\bm C\|_0$ as a measure of sparsity in formulation (\ref{LRSSC_gen}), we obtain the following nonconvex optimization problem:
\begin{equation} \label{LRSSC_l0}
 \begin{split}
\min_{\bm C}\frac{1}{2}\big\|\bm X&-\bm{XC}\big\|^2_F+\lambda \big\|\bm C\big\|_{S_0}+\tau \big\|\bm C\big\|_0\\&\ \  s.t.\ \ diag(\bm C)=\bm{0}.
\end{split}
\end{equation}
The proximity operator $H: \mathbb{R}\rightarrow  \mathbb{R}$ of $\|x\|_0$ is defined entry-wise as:
\begin{equation} \label{l0_prox}
H(y; \lambda)=arg\min_{x\in \mathbb{R}}\Big\{\frac{1}{2}(y-x)^2+\lambda \|x\|_0\Big\},
\end{equation}

The closed form solution of (\ref{l0_prox}) at $y\in \mathbb{R}$ is the hard thresholding function defined in (\ref{hard_thr}).
The proximity operator of $\big\|\bm C\big\|_{S_0}$ is the hard thresholding function applied entry-wise to the vector of singular values \cite{Le13, Liang16}.

\subsection{Optimization Algorithm}

To solve minimization problem in (\ref{LRSSC_l0}), 
we split original problem into two variables $\bm J$ and $\bm C$. That leads to the following objective function:
\begin{equation} \label{LRSSC_l0_reform}
\begin{split}
\min_{\bm J, \bm C}& \frac{1}{2}\big\|\bm X-\bm{XJ}\big\|^2_F+\lambda \|\bm C\|_{S_0}+\tau \|\bm C\|_0\ \  \\& s.t.\ \ \ \bm J=\bm C-diag(\bm C),
\end{split}
\end{equation}
The augmented Lagrangian function of (\ref{LRSSC_l0_reform}) is:
\begin{equation} \label{Lag2}
\begin{split}
&\mathcal{L}_\mu\big(\bm J,\bm C, \bm\Lambda\big) = \frac{1}{2}\big\|\bm X-\bm{XJ}\big\|^2_F+\lambda \|\bm C\|_{S_0}+\tau \|\bm C\|_0\\&+\frac{\mu}{2}\big\|\bm J-\bm C+diag(\bm C)\big\|^2_F+\big\langle\bm\Lambda, \bm J-\bm C+diag(\bm C)\big\rangle,
\end{split}
\end{equation}
where $\mu$ is penalty parameter and $\bm\Lambda$ is Lagrange multiplier.

\textit{Update rule for $\bm J^{k+1}$}: Given $\bm C^k$, $\bm\Lambda^k$, $\mu^k$, minimization of the Lagrangian function in (\ref{Lag2}) yields the following update:
\begin{equation} \label{J_l0_update}
\bm J^{k+1}= \big[\bm X^T\bm X+\mu^k\bm{I}\big]^{-1}\big[\bm X^T\bm X+\mu^k\bm C^k-\bm\Lambda^k\big].
\end{equation}

\textit{Update rule for $\bm C^{k+1}$:} Given $\bm J^{k+1}$, $\bm\Lambda^k$, $\mu^k$, the following problem needs to be solved:
\begin{equation} \label{C_min_l0}
\min_{\bm C}\lambda \big\|\bm C\big\|_{S_0}+\tau \big\|\bm C\big\|_0 +\frac{\mu^k}{2}\Big\|\bm J^{k+1}+\frac{\bm\Lambda^k}{\mu^k}-\bm C\Big\|_F^2.
\end{equation}
When $\lambda\neq 0$ and $\tau=0$, the proximal map reduces to:
\begin{equation} \label{prox1}
P_g^{\mu}=\argmin_{\bm C} \lambda \big\|\bm C\big\|_{S_0} +\frac{\mu^k}{2}\Big\|\bm J^{k+1}+\frac{\bm\Lambda^k}{\mu^k}-\bm C\Big\|_F^2.
\end{equation}
Let $\bm{U\Sigma V}^T$ denote the SVD of matrix $\big(\bm J^{k+1}+{\bm\Lambda}^k/\mu^k\big)$. The closed-form solution of (\ref{prox1}) is given by:
\begin{equation} \label{C1_update_l0}
{\bm C}^{k+1}=\bm{U}H\Big(\bm \Sigma; \frac{\lambda}{\mu^k}\Big)\bm{V}^T,
\end{equation}
where $H$ is the hard thresholding function defined in (\ref{hard_thr}) and applied entry-wise to $\bm \Sigma$.
Similarly, when $\lambda=0$ and $\tau \neq 0$ the proximal map is given by:
\begin{equation} \label{prox2}
P_f^{\mu}=\argmin_{\bm C} \tau \big\|\bm C\big\|_{0} +\frac{\mu^k}{2}\Big\|\bm J^{k+1}+\frac{\bm\Lambda^k}{\mu^k}-\bm C\Big\|_F^2.
\end{equation}
Closed-form solution of (\ref{prox2}) is obtained by the hard thresholding operator $H$ applied entry-wise to matrix $\big(\bm J^{k+1}+{\bm\Lambda^k}/\mu^k\big)$:
\begin{equation} \label{C2_update_l0}
\begin{split}
\bm C^{k+1}=H\Big(\bm J^{k+1}+\frac{{\bm\Lambda}^k}{\mu^k}; \frac{\tau}{\mu^k}\Big),\\ {\bm C}^{k+1}\leftarrow{\bm C}^{k+1}-diag\big({\bm C}^{k+1}\big).
\end{split}
\end{equation}

Proximal average, introduced recently in \cite{YuYao13} and generalized to nonconvex and nonsmooth setting in \cite{YuXun15, LinWei16}, allows us to efficiently solve problem in (\ref{C_min_l0})
when $\lambda\neq 0$ and $\tau\neq 0$. In particular, given that the proximal maps $P_f^{\mu}$ and $P_g^{\mu}$ can be easily solved using the hard thresholding operator, we approximate the proximal map $P_{f+g}^{\mu}$ by averaging solutions of proximal maps of low-rank and sparse regularizers:
\begin{equation}\label{prox_avg}
P_{f+g}^{\mu}\approx \lambda P_g^{\mu}+\tau P_f^{\mu},
\end{equation}
where parameters $\tau$ and $\lambda$ are set such that $\tau+\lambda=1$. 

Furthermore, since $\ell_0$ and $S_0$ norms belong to the class of semi-algebraic functions \cite{Attouch13}, the proximal average function $P_{f+g}$ is also a semi-algebraic function \cite{YuXun15}.

\textit{Update rule for Lagrange multiplier $\bm\Lambda^{t+1}$:} Given $ \bm J^{k+1}$, $\bm C^{k+1}$, $\mu^k$, Lagrange multiplier is updated with the following equation:
\begin{equation}\label{Lambda_rule_l0}
{\bm\Lambda^{k+1}}={\bm\Lambda^k}+\mu^k\big(\bm J^{k+1}-\bm C^{k+1}\big)\\
\end{equation}
The main steps of the proposed algorithm are summarized in Algorithm \ref{alg2}.

\begin{algorithm}[h]
\caption{$S_0$/$\ell_0$-LRSSC by ADMM optimization}
\label{alg2}
\begin{algorithmic}[1]
\REQUIRE Data points as columns in $\bm X $ 
, $\{\tau, \lambda\}>0$, $\tau+\lambda=1$
\ENSURE Assignment of the data points to $k$ clusters
\STATE Initialize: $\big\{\bm J,\bm C,\bm\Lambda \big\}=\bm 0$,  $\mu^{(0)}>0$, $\rho>1$ 
\STATE Compute $\bm X^T\bm X$ for later use
\WHILE{not converged}
\STATE Update $\bm J^{k+1}$ by (\ref{J_l0_update})
\STATE Normalize columns of $\bm J$ to unit $\ell_2$ norm\
\STATE Calculate rank regularized proximal map $P_g^\mu$ by (\ref{C1_update_l0}) 
\STATE Calculate sparsity regularized proximal map $P_f^\mu$ by (\ref{C2_update_l0})
\STATE Update $\bm C^{k+1}=P_{f+g}^\mu$ defined in (\ref{prox_avg}), (\ref{prox1}), (\ref{prox2})
\STATE Update $\bm\Lambda^{k+1}$ by (\ref{Lambda_rule_l0})
\STATE Update $\mu^{k+1}=\min(\rho\mu^{k}, \mu^{max})$
\ENDWHILE
\STATE Calculate affinity matrix $\bm W=|\bm C_{1}|+|\bm C_{1}|^T$\\
\STATE Apply spectral clustering \cite{Ng01} to $\bm W$\\
\end{algorithmic}
\end{algorithm}

\subsection{Convergence Analysis}

\begin{theorem} \label{thm2} 
	Let $Y^k= \big\{\big(\bm J^k, \bm C^k, \bm \Lambda^k\big)\big\}_{k=1}^\infty$ be a sequence generated by Algorithm 2. Then, for any sufficiently large $\mu$, Algorithm 2 converges globally\footnote{That is, regardless of the initialization, it generates a bounded sequence that has at least one limit point which is a stationary point of (\ref{Lag2}).}.
\end{theorem}

\begin{proof}
The results in \cite{YuXun15} guarantee convergence of the proximal average method. To guarantee global convergence of the Algorithm 2, we rewrite the problem (\ref{LRSSC_l0_reform}) using the following more general form:
\begin{equation} \label{LRSSC_l0_convg_form}
\begin{split}
&\min_{ \bm C, \bm J} f_1(\bm C)+f_2(\bm J)\\
& \text{subject to } \bm{AC}=\bm {BJ},
\end{split}
\end{equation}
where $\bm A=\bm I$, $\bm B=\bm I$, $f_1(\bm C)=\lambda \|\bm C\|_{S_0}+\tau \|\bm C\|_0$, $f_2(\bm J)=\frac{1}{2}\big\|\bm X-\bm{XJ}\big\|^2_F$.

We will now show that the assumptions A1-A5 in \cite{WangYin15} which guarantee convergence in nonconvex nonsmooth optimization problem are satisfied. 
$\|\cdot\|_0$ and $\|\cdot\|_{S_0}$ are nonnegative lower semi-continuous functions and lower bounded. Therefore, $f_1$ as a sum of these functions is also lower semi-continuous and lower bounded. Furthermore, $f_2$ is coercive and $\bm B = \bm I$, so assumptions A1 and A4 hold. $\bm A = \bm I$ and $\bm B = \bm I$ imply that assumptions A2 and A3 hold. Next, $f_2$ is Lipschitz differentiable function so assumption A5 is also satisfied. Therefore, A1-A5 are satisfied and Algorithm 2 converges for any sufficiently large $\mu$ \cite{WangYin15}. Of note, by splitting the original problem in (\ref{LRSSC_l0}) in three variables as done in GMC-LRSSC, we could not guarantee convergence since the assumption A2 would not be satisfied. 

Furthermore, $\ell_0$ and $S_0$ norms belong to the class of semi-algebraic functions and satisfy Kurdyka-\L{}ojasiewicz inequality \cite{Lojasiewicz1993, Attouch13}. Sum of semi-algebraic function is again a semi-algebraic function, so $\mathcal{L}_\mu\big(\bm J,\bm C, \bm\Lambda\big)$ in (\ref{Lag2}) is a semi-algebraic function and therefore, satisfies Kurdyka-\L{}ojasiewicz inequality. This allows us to establish stronger convergence property, that is, sequence $\big\{\big(\bm J^k, \bm C^k, \bm \Lambda^k\big)\big\}$ generated by Algorithm 2 converges regardless of the initialization to the unique limit point \cite{WangYin15}.
\end{proof}

\subsection{Stopping Criteria and Computational Complexity}

The steps in Algorithm 2 are repeated until convergence or when the maximum number of iterations is exceeded. The convergence is achieved when inequalities $\big\|\bm J^k - \bm C^k\|_{\infty}\leq\epsilon$ and $\big\|\bm J^k - \bm J^{k-1}\|_{\infty}\leq\epsilon$ are satisfied. In all experiments error tolerance $\epsilon$ is set to $10^{-4}$.

As in GMC-LRSSC, the computational complexity of Algorithm 2 is $O(nN^2+TN^3)$, where $T$ denotes the number of iterations. We set the maximal number of iterations $T$ to $100$, but the algorithm typically converged within $20$ iterations.

\section{Experimental Results}

In this section, we compare the clustering performance and efficiency of the proposed algorithms with the state-of-the-art subspace clustering algorithms on synthetic and four real-world datasets. The performance is evaluated in terms of clustering error (CE)
defined as:
\begin{equation}
CE(\hat{\bm{r}},\bm{r})=\min_{\pi\in\Pi_L}\Big(1-\frac{1}{N}\sum_{i=1}^{N}\mathbbm{1}_{\{\pi(\hat{\bm{r}}_i)=\bm{r}_i\}}\Big),
\end{equation}
where $\Pi_L$ is the permutation space of $[L]$. 

We compare the performance of our algorithms with the state-of-the-art subspace clustering algorithms, including Sparse Subspace Clustering (SSC) \cite{Elhamifar09}, Low-Rank Representation (LRR) \cite{Liu10b, Liu13}, closed form Low-Rank Subspace Clustering (LRSC) \cite{Favaro11}, Sparse Subspace Clustering via Orthogonal Matching Pursuit (SSC-OMP) \cite{Dyer13}, Thresholding based Subspace Clustering (TSC) \cite{Heckel15}, Nearest Subspace Neighbor (NSN) \cite{Park14}, Low-Rank Sparse Subspace Clustering (LRSSC) \cite{Wang13}, $\ell_0$-Sparse Subspace clustering ($\ell_0$-SSC) \cite{Yang16, Yang18s} and Schatten-$p$ norm minimization based LRR \cite{Zhang18} ($S_{2/3}$-LRR and $S_{1/2}$-LRR). 

\subsection{Experimental Setup}

In all experiments, we set the parameters of GMC-LRSSC, and $S_0$/$\ell_0$-LRSSC as follows: $\tau=1-\lambda$, $\rho=3$, $\mu^{max} = 10^6$, $\epsilon$ in stopping criteria to $10^{-4}$ and maximum number of iterations to $100$. Parameters $\lambda$ and initial value of $\mu$ are tuned more carefully. For GMC-LRSSC $\lambda$ is parameterized using $\alpha$ as $1/(1+\alpha)$, where $\alpha$ is tested in range $[10^{-3},10^{3}]$ with step $10$. Both $\lambda$ and $\tau$ are scaled by $\mu_2^0$. For $S_0/\ell_0$-LRSSC parameter $\lambda$ is optimized in range $[0.1,0.9]$ with step $0.1$. After the best $\lambda$ is found, $\mu_2^0$ in GMC-LRSSC and $\mu^0$ in $S_0/\ell_0$-LRSSC are tested in the set $\{1,3,5,10,20\}$. Initial value of parameter $\mu_1$ in GMC-LRSSC is set to $0.1$ in all experiments. For GMC-LRSSC we test nonconvexity parameter $\gamma \in \{0.1, 0.6, 1\}$. That resulted in $\gamma=1$ on the Extended Yale B dataset, $\gamma=0.6$ on the MNIST dataset and $\gamma=0.1$ on the USPS and ISOLET1 datasets. On synthetic data we test $\gamma$ from $0.1$ to $1$ with step $0.1$.

For other state-of-the-art algorithms, we use the source codes provided by the authors. If the best parameters are available, we set them as reported in the corresponding papers/source codes. Otherwise, we tuned the parameters and retained those with the best performance. 
Specifically, for SSC parameter $\alpha \in\{10, 20, 50, 80, 100, 200, 500, 800, 1000\}$, for LRR $\lambda \in\{0.05, 0.1, 0.3, 0.5, 1, 2, 3, 4, 5, 6, 7\}$, for LRSC $\tau \in \{ 0.1,5,10,20,50,80,100,200,500,800\}$ and $\alpha \in \{0.1\tau, 0.5\tau,0.9\tau,1.11\tau, 2\tau, 10\tau\}$, for $S_{2/3}$-LRR and $S_{1/2}$-LRR $\lambda \in \{0.01, 0.05,0.1,0.3,0.5,0.8,1,1.5, 2,3,5,10\}$, and $\lambda$ in $\ell_0$-SSC is tuned in range $[0.1,1]$ with step $0.1$. In NSN and SSC-OMP number of neighbors is chosen in the set $\{2, 3, 5, 8, 10, 12, 15, 18, 20\}$. For TSC we set $q=max(3,\lceil n/20\rceil).$ 
For LRSSC we test $\lambda$ parameter in range $[10^{-3},10^{4}]$ with step $10$ on real-world datasets. In order to have completely same setting on the synthetic data, we tune the LRSSC parameters in the same way as for GMC-LRSSC.

Parameters of all algorithms are tuned on $20$ runs and for $L=\{3,5,10\}$ with different random seed than in the final experiment. In the final experiment we run each algorithm $100$ times.

\subsection{Synthetic Data}

In the synthetic data experiment we compare LRSSC, GMC-LRSSC and $S_0/\ell_0$-LRSSC 
for different levels of noise and number of samples. We generate three $5$-dimensional disjoint subspaces embedded in the $100$-dimensional space. Subspace bases $\{U_i\}_{i=1}^3\in \mathbb{R}^{100\times 5}$ are constructed such that rank$([\bm U_1,\bm U_2,\bm U_3])=10$. We randomly sample $N_i$ data points from each subspace by computing $\{\bm X_i=\bm U_i\bm A_i\}_{i=1}^3$, where $\{\bm A_i\}_{i=1}^3 \in \mathbb{R}^{5\times N_i}$ is generated from $\mathcal N(0,1)$ distribution. We sample the same number of data points from each subspace, i.e. $N_1=N_2=N_3$. We then add Gaussian noise with zero mean and vary the noise variance. Fig. \ref{Fig2} shows the average clustering error over $10$ runs for different number of samples per subspace and different noise variance.


\begin{figure}[h]
	\centering
	\includegraphics[width=0.48\textwidth]{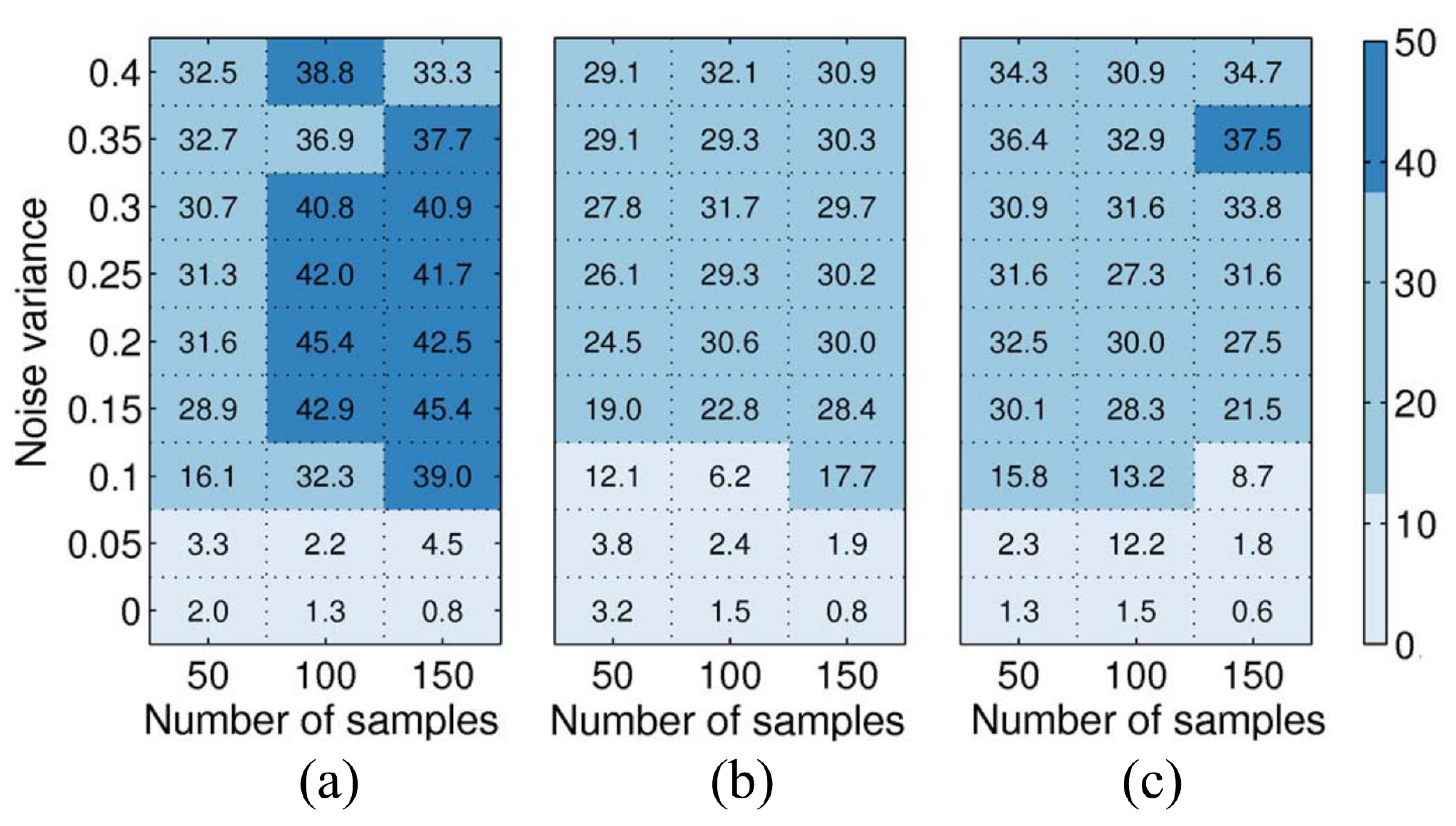}
	\caption{Clustering error ($\%$) on synthetic data when varying number of samples per subspace (x-axis) and noise variance (y-axis). (a) LRSSC. (b) GMC-LRSSC. (c) $S_0/\ell_0$-LRSSC.} 
	\label{Fig2}
\end{figure}
For $50$ data points per subspace and small measurement noise, $S_0/\ell_0$-LRSSC performs better than LRSSC and GMC-LRSSC. On the other hand, for larger measurement noise GMC-LRSSC is the best performing algorithm. When we increase the number of data poins to $100$, GMC-LRSSC remains the best performing algorithm for most levels of noise. However, when further increasing number of data points, $S_0/\ell_0$ performs better except for very large measurement noise. This is in line with results presented in \cite{Zheng17} which show that $\ell_0$ quasi-norm regularization of least-squares problems outperforms $\ell_p$ regularization $0<p\leq1$ for small measurement noise. Whereas LRSSC and GMC-LRSSC in most cases do not improve performance when increasing the number of data points, $S_0/\ell_0$ is often able to exploit additional data.

\subsection{Face Recognition Dataset}

The Extended Yale B dataset \cite{Georghiades01, Lee05} consists of face images of 38 individuals (subjects). It contains 64 frontal face images of each individual acquired under different illumination conditions. We use down-sampled $48\times 24$ pixel images and consider each vectorized image as one data point. The face images of each individual in Yale B dataset lie approximately in a 9-dimensional subspace \cite{Elhamifar13}.

\begin{table*}[!h]
\fontsize{8}{8}\selectfont
\renewcommand{\arraystretch}{1.3}
\caption{Clustering error (\%) on the Extended Yale B dataset}
\label{yaleb_acc}
\centering
\begin{tabular}{ | c || c | c | c | c | c | c | c| c | c | c |  c | c | } 
\hline
  L & SSC & LRR & LRSC & SSC-OMP & TSC & NSN & LRSSC & $\ell_0$-SSC  & $S_{2/3}$-LRR & $S_{1/2}$-LRR & GMC-LRSSC & $S_0$/$\ell_0$-LRSSC \\\hline\hline
	5 &  4.54 & 4.38 & 14.56 & 7.46 & 28.38 & 5.42 & 21.19 & 16.27 &  8.94 & 8.92 & \textbf{3.97} & \textbf{3.52} \\\hline
	10  & 8.78 & 7.80 & 34.74 & 15.26 & 39.73 & 9.19 & 26.89 & 28.45 & 9.45 & 9.49 & \textbf{4.00} & \textbf{4.45} \\\hline
	20   &  21.52 & 16.68 & 28.23 & 17.23 & 45.57 & 15.02 & 36.78 & 39.24 & 11.98 & 11.58 & \textbf{6.38} & \textbf{7.14} \\\hline
	30   &  26.73 & 21.27 & 31.53 & 20.53 & 47.10 & 18.69 & 33.60 & 39.54 & 12.02 & 11.61 &\textbf{8.65} & \textbf{8.35} \\\hline
\end{tabular}
\end{table*}

We perform experiments for different number of clusters, ranging from 5 to 30. In each experiment we sample uniformly $L$ clusters from the total number of subjects and compute average of clustering error over 100 random subsets. The results are reported in Table \ref{yaleb_acc} with the two best results highligted in boldface. Compared to the state-of-the-art methods, GMC-LRSSC and $S_0/\ell_0$-LRSSC achieve the lowest clustering error on all four clustering tasks. The difference between our algorithms and the second best performing other is significant for 10, 20 and 30 clusters (FDR<$1\%$; Benjamini-Hochberg corrected). Importantly, by increasing the number of clusters, the difference between our $\ell_0$-based formulations and convex low-rank and sparse formulations becomes larger. 

To check the effect of parameter $\gamma$ in GMC-LRSSC, we vary parameter $\gamma$ from $0.1$ to $1$, where small $\gamma$ means that GMC penalty $\psi_B$
is close to convex, and $\gamma=1$ corresponds to maximally nonconvex value of the penalty. Fig. \ref{Fig3} shows performance as a function of $\gamma$ values for $10$, $20$ and $30$ clusters. On all these tasks, larger values of $\gamma$ achieve lower clustering error than the smaller values.

\begin{figure}[h]
	\centering
	\includegraphics[width=0.34\textwidth]{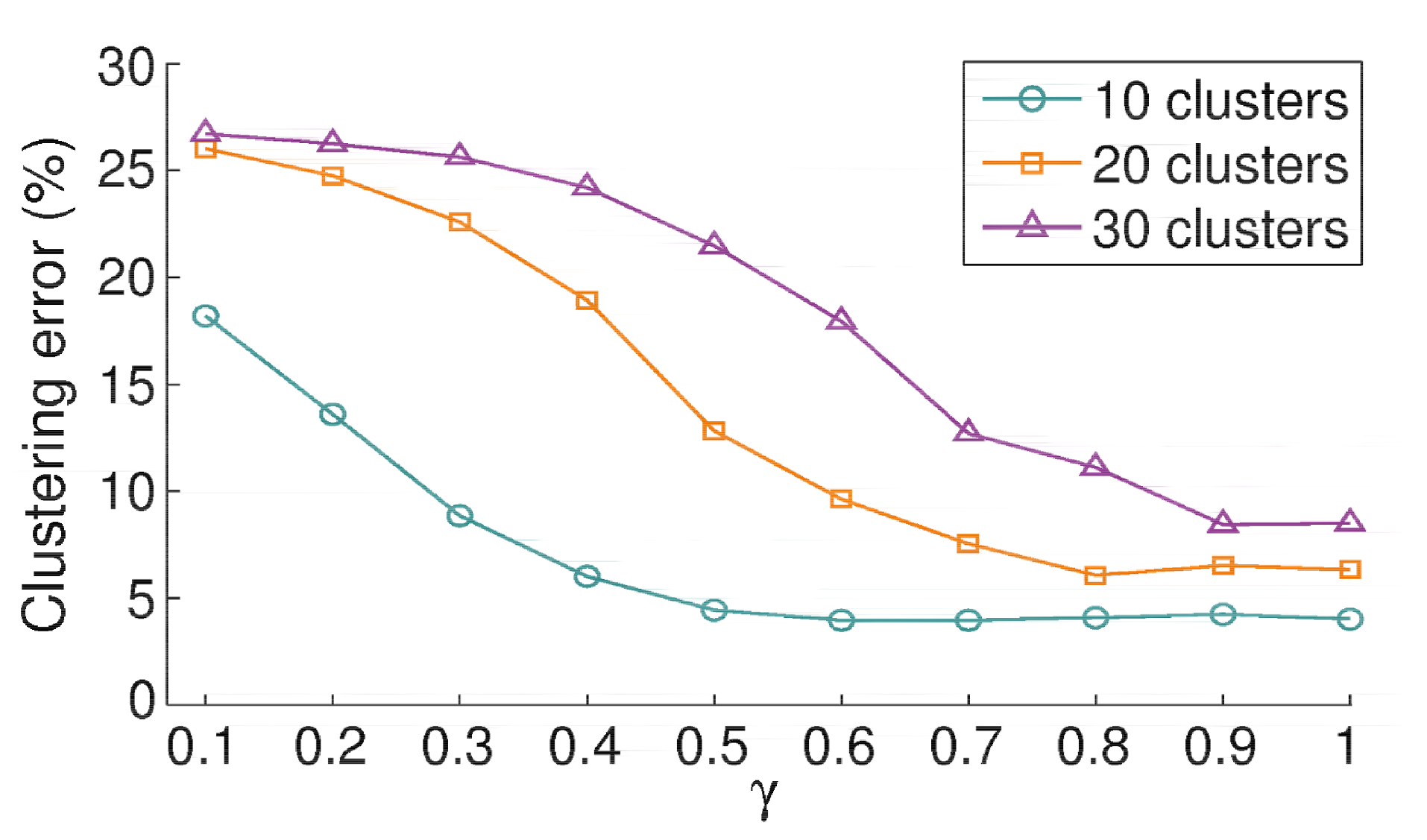}
	\caption{Clustering error on the Extended Yale B dataset for $10$, $20$ and $30$ clusters when varying $\gamma$ parameter in GMC-LRSSC.} 
	\label{Fig3}
\end{figure}

\subsection{Handwritten Digit Datasets}

For handwriting recognition task, we consider two datasets: MNIST and USPS 
datasets. Both datasets contain pictures of ten digits (0-9), each digit corresponding to one cluster. The MNIST dataset contains $10000$ centered $28\times 28$ pixel images of handwritten digits. The USPS dataset consists of 92898 handwritten digit images, each of $16\times 16$ dimension. The handwritten digits lie approximately in a 12-dimensional subspace \cite{Hastie98}. 

For both datasets, we use a subset of available images, sampling uniformly 50 images per digit in each run and compute average of clustering error over 100 runs. The performance comparisons for different choice of digits are shown in Table \ref{mnist_usps_acc}. On both datasets, $S_0/\ell_0$-LRSSC and GMC-LRSSC are the only algorithms that consistently achieve high performance across varying combinations of digits.

On the MNIST dataset $S_0/\ell_0$-LRSSC algorithm is among the best performing algorithms, and the difference is increasing for larger number of clusters. GMC-LRSSC has lower performance than $S_0/\ell_0$-LRSSC for three combinations, but it is still significantly better than SSC, LRR and LRSSC. SSC and LRSSC have the best performance for digit sets $\{2,4,8\}$ and $\{2,4,6,8,9\}$, but they fail to give satisfactory results on other combinations of digits. 

On the USPS dataset GMC-LRSSC is slightly better than $S_0/\ell_0$-LRSSC, except for the combinations of five digits. Specifically, on the digit set $\{2,4,6,8,9\}$ $S_0/\ell_0$-LRSSC outperforms all other methods. 
For the hardest problems with $10$ clusters GMC-LRSSC and $S_0/\ell_0$-LRSSC again have significantly better performance (FDR<$1\%$) than all other methods, that is $7.2\%$ and $4.8\%$ higher than the second best method, respectively. Fig. \ref{Fig4} illustrates derived affinity matrices on the USPS dataset for 10 clusters.

\begin{figure}[h]
	\centering
	\includegraphics[width=0.37\textwidth]{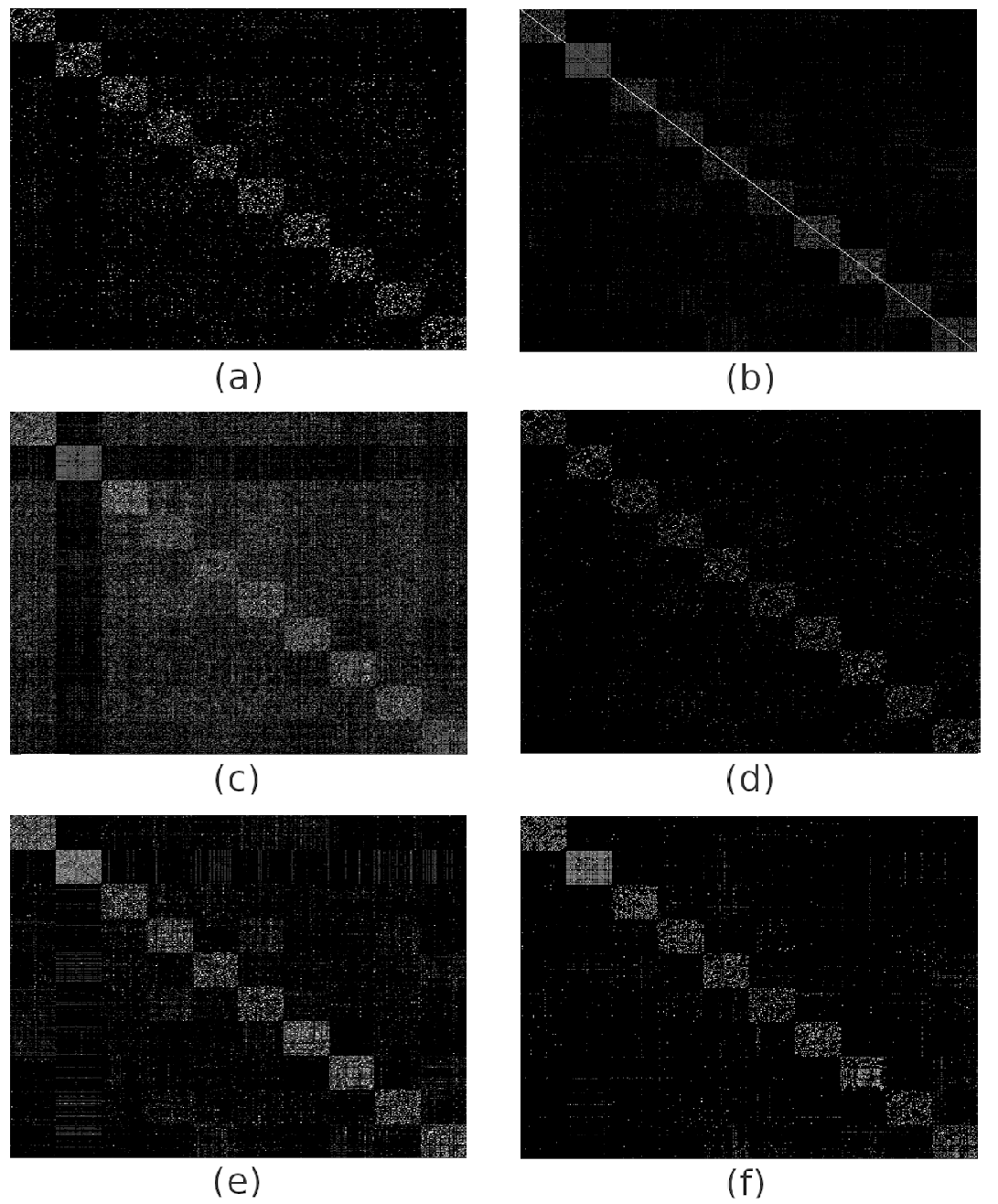}
	\caption{Visualization of affinity matrices on the USPS dataset using all $10$ digits. (a) SSC. (b) LRR. (c) LRSSC. (d) $\ell_0$-SSC. (e) GMC-LRSSC. (f) $S_0/\ell_0$-LRSSC.} 
	\label{Fig4}
\end{figure}

\begin{table*}[h]
\fontsize{8}{8}\selectfont
\renewcommand{\arraystretch}{1.3}
\caption{Clustering error (\%) on MNIST and USPS datasets}
\label{mnist_usps_acc}
\centering
\begin{tabular}{ | c | c || c | c | c | c | c | c| c | c | c | c | c | c| } 
\hline
\multirow{2}{*} {Digits} & \multirow{2}{*}{Dataset} & \multirow{2}{*}{SSC} &\multirow{2}{*}{LRR} & \multirow{2}{*}{LRSC} & SSC- & \multirow{2}{*}{TSC} & \multirow{2}{*}{NSN} & \multirow{2}{*}{LRSSC} & \multirow{2}{*}{$\ell_0$-SSC} & $S_{2/3}$- & $S_{1/2}$- & GMC- & $S_0$/$\ell_0$- \\ & & & & & OMP & & & & & LRR & LRR & LRSSC & LRSSC \\\hline\hline
	 \multirow{2}{*}{2,4,8} & MNIST &\textbf{7.43} & 14.14 & 10.59 & 11.06 & 12.09 & 13.02 & \textbf{7.01} & 7.82 & 14.80 & 15.03 & 8.66 & 8.92\\
	 & USPS & 6.02 & 10.37 & 7.61 &21.04 & 8.32 & 18.67 & 7.13 & \textbf{4.88} & 8.83 & 9.70 & \textbf{4.92} & 6.40 \\\hline
	\multirow{2}{*}{3,6,9} & MNIST & 3.89 & 3.49 & 4.61 & 5.69 & 3.25 & \textbf{2.99} & 4.15 & 3.49 & 5.93 & 6.61 & \textbf{2.93} & 3.25\\
	& USPS & 2.05 & 1.57 & 4.50 & 23.43 & 1.43 & 2.22 & 8.13 & \textbf{1.09} & 3.02 & 3.68 & \textbf{0.97} & 1.27\\\hline
	\multirow{2}{*}{1,4,7} & MNIST &  47.50 & 45.09 & 44.14 & 42.21 & 33.72 & \textbf{14.05} & 47.09 & 45.40 & 43.42 & 42.57 & 34.50 & \textbf{27.33} \\
	& USPS  & \textbf{2.21} & 4.01 & 4.33 & 58.07 & 7.45 & 9.45 & 8.61 & 3.68 & 4.27 & 5.37 & \textbf{2.65} & 3.82 \\\hline
	\multirow{2}{*}{2,4,6,8,9} & MNIST & \textbf{25.54} & 33.54 & 29.22 & 29.43 & 28.95 & 27.04 & \textbf{26.78} & 28.94 & 34.90 & 36.29 & 27.40 & 27.20 \\
	& USPS  & 15.69 & 22.30 & 18.86 & 53.43 & 20.35 & 26.69 & 18.67 & \textbf{14.31} & 19.22 & 19.95 & 15.37 & \textbf{13.38} \\\hline
	\multirow{2}{*}{0,1,3,5,7} & MNIST & 53.60 & 33.61 & 35.92 & 37.51 & 30.16 & \textbf{22.39} & 46.86 & 33.74 & 32.50 & 33.08 & 29.80 & \textbf{27.85} \\
	& USPS  & 30.00 & 22.83 & 28.47 & 74.66 & 25.58 & \textbf{13.36} & 35.44 & 30.17 & 27.41 & 27.68 & 24.76 & \textbf{10.89} \\\hline
	\multirow{2}{*}{0-9} & MNIST & 47.49 & 45.13 & 46.88 & 46.45 & 40.00 & \textbf{34.81} & 45.68 & 39.51 & 43.19 & 43.66 & 38.01 & \textbf{34.89} \\
	& USPS  & 28.28 & 33.44 & 28.58 & 84.01 & 29.34 & 28.43 & 33.10 & 27.27 & 28.75 & 29.24 & \textbf{20.46} & \textbf{22.45} \\\hline
\end{tabular}
\end{table*}

\subsection{Speech Recognition Dataset}

For the speech recognition task, we evaluate algorithms on the ISOLET dataset \cite{Fanty91}. The task is to cluster subjects, where each subject spoke the name of each letter of the alphabet twice. We use dataset ISOLET1 containing 26 subjects with 30 data points from each subject. The features include spectral coefficients, contour features, sonorant, presonorant and postsonorant features. To check whether the subspace clustering assumption holds on the ISOLET1 dataset, we compute the singular values of several subjects. Fig. \ref{Fig5} demonstrates that singular values decay rapidly and confirms that data points are drawn from low-dimensional subspaces. 

\begin{figure}[h]
		\centering
		\includegraphics[width=0.33\textwidth]{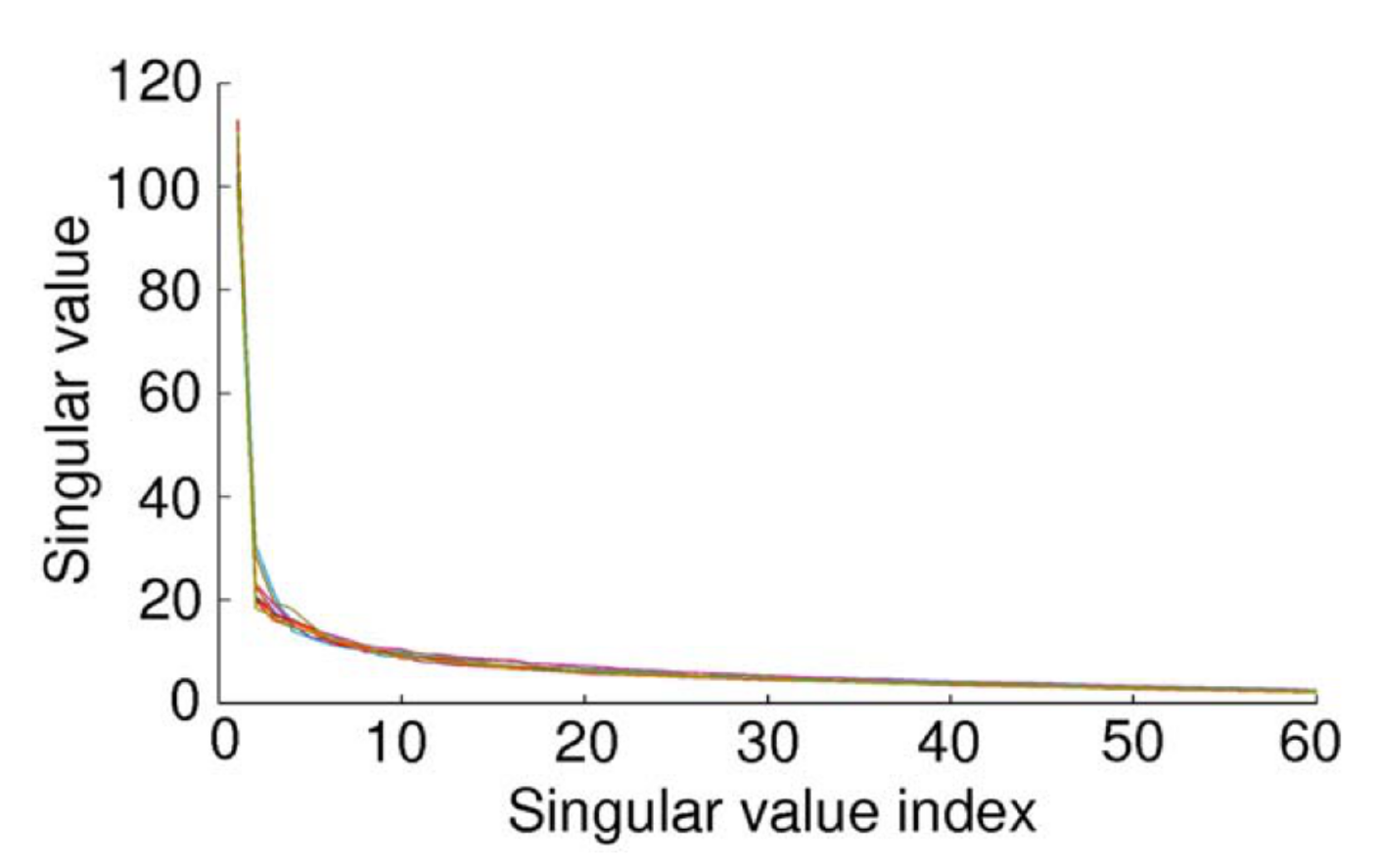}
		\caption{Singular values of several speakers on the ISOLET1 dataset.} 
		\label{Fig5}
\end{figure}

\begin{table*}[!t]
\fontsize{8}{8}\selectfont
\renewcommand{\arraystretch}{1.3}
\caption{Clustering error (\%) on the ISOLET1 dataset}
\label{isolet1_acc}
\centering
\begin{tabular}{ | c || c | c | c | c | c | c| c | c | c | c | c | c |} 
\hline
  L & SSC & LRR & LRSC & SSC-OMP & TSC & NSN & LRSSC & $\ell_0$-SSC & $S_{2/3}$-LRR & $S_{1/2}$-LRR & GMC-LRSSC & $S_0$/$\ell_0$-LRSSC \\\hline\hline
	5 & 10.98 & 7.61 & 10.25 & 27.79 & 11.58 & 8.23 & 8.63 & 9.23 & 7.45  & 8.06 & \textbf{7.07} &  \textbf{6.87}\\\hline
	10  & 17.11 & 14.65 & 15.84 & 44.44 & 19.35 & 16.04 & 14.72 & 18.21 & 14.10 & 14.34 & \textbf{13.92} & \textbf{13.81} \\\hline
	15  &  25.64 & 23.14 & 22.88 & 54.03 & 27.07 & 23.61 & 23.87 & 25.73 &  20.37 & \textbf{20.24} & 20.29 & \textbf{19.90}  \\\hline
	20    &  31.05 & 30.60 & 27.93 & 59.89 & 31.70 & 27.94 & 29.90 & 30.28  & 26.02 & \textbf{25.24} & 25.32 & \textbf{25.07} \\\hline
\end{tabular}
\end{table*}

We sample uniformly $L=\{5,10,15,20\}$ clusters from the total number of subjects over 100 random subsets. The average clustering errors are reported in Table \ref{isolet1_acc}. For all tested numbers of clusters $S_0/\ell_0$-LRSSC is the best performing method. GMC-LRSSC is the second best method for $5$ and $10$ clusters, while for $15$ and $20$ it achieves the same result as Schatten-$2/3$ and Schatten-$1/2$ LRR. 

\subsection{Computational Time and Convergence}

We further test the convergence behavior of GMC-LRSSC and $S_0/\ell_0$-LRSSC. The convergence conditions of GMC-LRSSC are satisfied within less than 15 iterations on all four real-world datasets. Fig. \ref{Fig6} illustrates convergence behavior of GMC-LRSSC on the MNIST and ISOLET1 datasets for 10 clusters. $S_0/\ell_0$-LRSSC converges within 20 iterations on the Extended Yale B, MNIST and USPS datasets. Fig. \ref{Fig7} shows behavior of $S_0/\ell_0$-LRSSC on the MNIST and ISOLET1 datasets for 10 clusters. Although the maximal number of iterations is exceeded on the ISOLET1 dataset, Fig. \ref{Fig7} shows that the error decays rapidly and within 20 iterations. 

The average computational time over 100 runs of each algorithm is shown in Table \ref{avg_time}. On all datasets, LRSC is consistently the fastest algorithm. GMC-LRSSC is among the fastest algorithms, whereas $S_0/\ell_0$-LRSSC is among the fastest algorithms on the Extended Yale B and MNIST datasets. As explained above, on ISOLET1 dataset $S_0/\ell_0$-LRSSC exceeds maximum number, resulting in higher computational time. 
All experiments were run in MATLAB 2017a environment on the PC with $256$ GB of RAM and Intel Xeon CPU E5-2650 v4 $2$ processors operating with a clock speed of $2.2$ GHz.

\begin{figure}[h]
  		\centering
  		\includegraphics[width=0.4\textwidth]{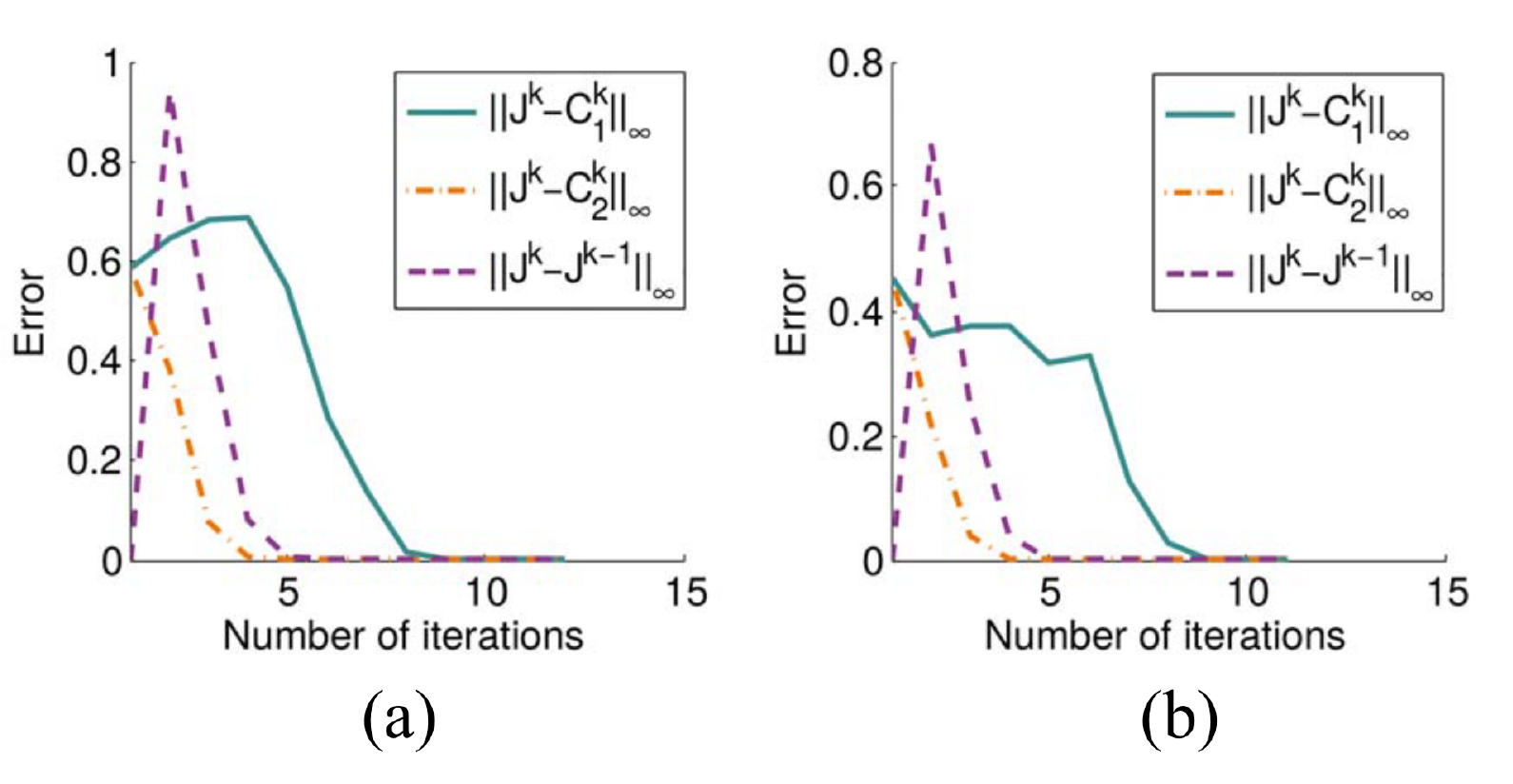}
  		\caption{Convergence of GMC-LRSSC. (a) MNIST dataset. (b) ISOLET1 dataset.}
  		\label{Fig6}
\end{figure}
 
\begin{figure}[h]
  		\centering
  		\includegraphics[width=0.4\textwidth]{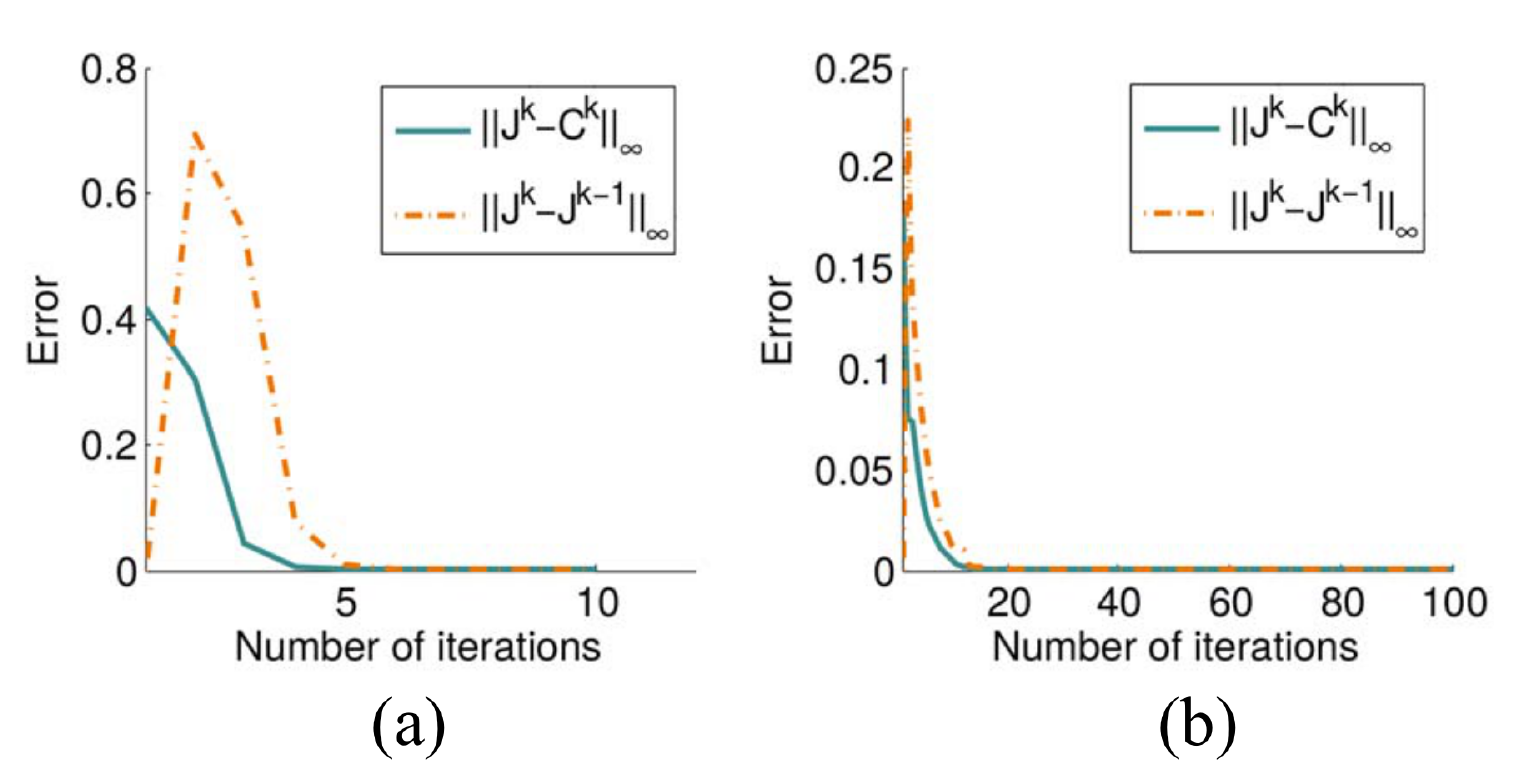}
  		\caption{Convergence of $S_0/\ell_0$-LRSSC. (a) MNIST dataset. (b) ISOLET1 dataset. }
  		\label{Fig7}
\end{figure}

\begin{table*}[!t]
\fontsize{8}{8}\selectfont
\renewcommand{\arraystretch}{1.3}
\caption{Average time (s) on the Extended Yale B, MNIST and ISOLET1 datasets}
\label{avg_time}
\centering
\begin{tabular}{ | c | c || c | c | c | c | c | c| c | c | c | c | c | c| } 
\hline
\multirow{2}{*} {Dataset} & \multirow{2}{*}{L} & \multirow{2}{*}{SSC} &\multirow{2}{*}{LRR} & \multirow{2}{*}{LRSC} & SSC- & \multirow{2}{*}{TSC} & \multirow{2}{*}{NSN} & \multirow{2}{*}{LRSSC} & \multirow{2}{*}{$\ell_0$-SSC} & $S_{2/3}$- & $S_{1/2}$- & GMC- & $S_0$/$\ell_0$- \\ & & & & & OMP & & & & & LRR & LRR & LRSSC & LRSSC \\\hline\hline
  	 \multirow{3}{*}{Yale B} & 5 & 30.88 & 10.21 & \textbf{0.18} & 0.56 & 0.44 & 1.28 & 10.65 & 40.11 & 9.85 & 10.02 & \textbf{0.39} & 0.59 \\
  	 & 10 & 54.30 & 31.52 & \textbf{0.41} & 1.29 & \textbf{1.00} & 2.94 & 69.97 & 101.00 & 33.78 & 34.32 & 1.86 & 2.91\\
  	& 30 & 180.67 & 271.52 & \textbf{2.47} & 6.29 & \textbf{6.20} & 14.84 & 642.33 & 623.94  & 295.74 & 297.36 & 19.19 & 29.24 \\\hline
  	\multirow{3}{*}{MNIST} & 3 & 0.56 & 1.17 & \textbf{0.05} & 0.15 & 0.17 & 0.52 & 0.37 & 2.63 & 1.35 & 1.37 & 0.12 & \textbf{0.11}\\
  	 & 5 & 1.31 & 4.01 & \textbf{0.10} & 0.27 & 0.29 & 0.85 & 1.40 & 5.11 & 5.03 & 5.10 & 0.30 & \textbf{0.26}\\
  	& 10 & 4.59 & 13.56 & \textbf{0.26} &\textbf{0.60} & 0.67 & 1.83 & 12.44 & 14.37 & 17.09 & 17.41 & 1.45 & 1.18
  	 \\\hline
  	\multirow{3}{*}{ISOLET1} & 5 & 6.83 & 4.78 & \textbf{0.09} & \textbf{0.20} & 0.25 & 0.60 & 4.24 & 4.71 &  5.99 & 6.12 & 0.41 & 2.36\\
  	 & 10 & 14.23 & 17.23 & \textbf{0.26} & \textbf{0.47} & 0.61 & 1.31 & 30.05 & 12.82 & 21.78 &  22.21 & 1.87 & 12.95\\
  	& 20 & 38.37 & 30.34 & \textbf{0.28} & \textbf{1.17} & 1.64 & 3.05 & 133.86 & 45.79 & 34.05 & 34.87 & 6.96  & 50.93 \\\hline	
\end{tabular}
\end{table*}

\section{Conclusion and Discussion}
In this paper we have introduced two nonconvex regularizations in low-rank sparse subspace clustering: (i) multivariate generalization of minimax-concave penalty function; and (ii) $S_0/\ell_0$ regularization with a proximal operator approximated using the proximal average method. Under the proposed framework, we have presented two algorithms based on the alternating direction method of multipliers. We have performed extensive experiments on synthetic and four real world datasets, including face recognition, handwriting recognition and speech recognition tasks. Our experimental results have shown that both proposed methods converge fast and achieve clustering error lower than nuclear and $\ell_1$ norms regularized objective. Moreover, for larger number of clusters proposed methods consistently outperform existing subspace clustering methods. That is explained by their more accurate approximation of rank and sparsity of the data representation matrix than it is the case with nuclear and $\ell_1$ norms.

The choice of norm should be decided depending on the dataset. In future work we plan to study how different initializations affect the accuracy of the proposed nonconvex regularization based methods. Instead of directly solving NP hard $\ell_0$ quasi-norm minimization problem, we are interested in finding a way to gradually build a solution. A possible strategy could be to start with $\ell_1$ norm solution. Since analytic formulas for thresholding function exist for $p=1/2$ and $p=2/3$ \cite{Suter16}, we can gradually shrink $p$ and use the current solution to initialize the next step of the algorithm. This approach is called $p$-continuation strategy in \cite{Zheng17}. 

\appendix[Proof of Proposition \ref{boundness} and Theorem \ref{thm1}]

In this section, we first prove the boundedness of variables in Algorithm 1. This result helps us to establish convergence property of Algorithm 1. We then prove Theorem \ref{thm1} in the paper, where we show that any converging point satisfies Karush-Kuhn-Tucker (KKT) conditions \cite{Kuhn51}.\\

\textit{Proof of Proposition \ref{boundness}}:\\
From the first order optimality conditions of Lagrangian function in (\ref{Lag}) we have:
\begin{equation}
\begin{split}
& \bm 0\in \partial_{\bm C_1}\mathcal{L}_{\mu_1^k,\mu_2^k}\big(\bm J^{k+1},\bm C_1^{k+1},\bm C_2^{k+1}, \bm\Lambda_1^k,\bm\Lambda_2^k\big)\\
& \bm 0\in \partial_{\bm C_2}\mathcal{L}_{\mu_1^k,\mu_2^k}\big(\bm J^{k+1},\bm C_1^{k+1},\bm C_2^{k+1}, \bm\Lambda_1^k,\bm\Lambda_2^k\big).
\end{split}
\end{equation}
The optimality condition of problem in (\ref{C2_min_gmc}) implies that:
\begin{equation} \label{opt_cond_C2}
\big[\partial \tau\psi_{\bm B}\big(\bm C_2^{k+1}\big)\big]_{ij}-\big[\bm\Lambda_2^{k+1}\big]_{ij}=0,
\end{equation}
where $\big[\partial \psi_{\bm B}\big(\bm C_2^{k+1}\big)\big]_{ij}$ denotes gradient of the GMC penalty $ \psi_{\bm B}$ at $\big[\bm C_2^{k+1}\big]_{ij}$.
By the definition of the scaled MC penalty in (\ref{scaled_MC_penalty}) and using $\bm B=\sqrt{\mu_2^k\gamma/\tau}\bm{I}$ we have:
\begin{equation} \label{phi_partial}
\partial \phi_ b\big(c_{ij}\big)=
\begin{cases}
sign(c_{ij})-\frac{\mu_2^k\gamma}{\tau} c_{ij}, & \text{if}\ |c_{ij}|\leq \frac{\tau}{\mu_2^k\gamma}\\
0, & \text{if}\ |c_{ij}|> \frac{\tau}{\mu_2^k\gamma},
\end{cases}
\end{equation}
where $c_{ij}$ denotes $\big[\bm C_2\big]_{ij}$.\\
If $|c_{ij}|> \tau/(\mu_2^k\gamma)$, then from (\ref{opt_cond_C2}) and (\ref{phi_partial}) directly follows $\big[\bm\Lambda_2^{k+1}\big]_{ij}=0$.\\
Otherwise, we get the following equality:
\begin{equation}
\big[\bm\Lambda_2^{k+1}\big]_{ij}= sign\big(\big[\bm C_2^{k+1}\big]_{ij}\big)-\frac{\mu_2^k\gamma}{\tau} \big[\bm C_2^{k+1}\big]_{ij}.
\end{equation}
Since $\big|\big[\bm C_2^{k+1}\big]_{ij}\big|\leq \tau/(\mu_2^k\gamma)$, it follows $\big|\big[\bm\Lambda_2^{k+1}\big]_{ij}\big|\leq 1$. Therefore, sequence $\big\{\bm \Lambda_2^{k}\big\}$ is bounded.

The optimality condition of problem in (\ref{C1_min_gmc}) implies that:
\begin{equation} \label{opt_cond_C1}
\big[\partial \lambda \psi_{\bm B}\big(\sigma\big(\bm C_1^{k+1}\big)\big)\big]_{ij} -\big[\bm\Lambda_1^{k+1}\big]_{ij}=0,
\end{equation}
Similarly, following the proof for $\bm C_2$ and using Proposition \ref{propos1}, it can be shown that the sequence $\big\{\bm \Lambda_1^{k}\big\}$ is also bounded. Using the definitions of $\bm J^{k+1}$, $\bm C_1^{k+1}$ and $\bm C_2^{k+1}$ as minimizers, we have the following inequalities \cite{Shang18}:
\begin{equation} \label{lag_def_ineq}
\begin{split}
\mathcal{L}_{\mu_1^k, \mu_2^k}\big(\bm J^{k+1},\bm C_1^{k+1},\bm C_2^{k+1},& \bm\Lambda_1^{k},\bm\Lambda_2^{k}\big)\\&\leq \mathcal{L}_{\mu_1^k, \mu_2^k}\big(\bm J^{k+1},\bm C_1^{k+1},\bm C_2^k, \bm\Lambda_1^{k},\bm\Lambda_2^{k}\big) \\
&\leq \mathcal{L}_{\mu_1^k, \mu_2^k}\big(\bm J^{k+1},\bm C_1^{k},\bm C_2^k, \bm\Lambda_1^{k},\bm\Lambda_2^{k}\big)\\
&\leq \mathcal{L}_{\mu_1^k, \mu_2^k}\big(\bm J^{k},\bm C_1^{k},\bm C_2^k, \bm\Lambda_1^{k},\bm\Lambda_2^{k}\big).
\end{split}
\end{equation}
Note that the last term equals:
\begin{equation} \label{lag_def_eq}
\begin{split}
&\mathcal{L}_{\mu_1^k, \mu_2^k}\big(\bm J^{k},\bm C_1^{k},\bm C_2^k, \bm\Lambda_1^{k},\bm\Lambda_2^{k}\big)\\&=\mathcal{L}_{\mu_1^{k-1},\mu_2^{k-1}}\big(\bm J^{k},\bm C_1^{k},\bm C_2^k, \bm\Lambda_1^{k-1},\bm\Lambda_2^{k-1}\big)+\sum_{i=1}^{2}a_i\|\bm \Lambda_i^k-\bm\Lambda_i^{k-1}\|_F^2,
\end{split}
\end{equation}
where $a_i=\frac{\mu_i^k+\mu_i^{k-1}}{2{(\mu_i^{k-1})}^2}$, $i=1,2$.
Since $\mu^k=\rho\mu^{k-1}$, $\rho > 1$, it follows that $\mu^k$ is non-decreasing and $\sum_{k=0}^{\infty}\frac{1}{\mu_i^k} <\infty$, $i=1,2$. We then have:
\begin{equation} \label{sum_mu}
\sum_{k=1}^{\infty}\frac{\mu_i^k+\mu_i^{k-1}}{2{(\mu_i^{k-1})}^2}=\sum_{k=1}^{\infty}\frac{1+\rho}{2\mu_i^{k-1}}<\infty.
\end{equation}
Since $\big\{\bm \Lambda_1^{k}\big\}$ and $\big\{\bm \Lambda_2^{k}\big\}$ are bounded, then $\big\{\|\bm \Lambda_1^k-\bm\Lambda_1^{k-1}\|_F^2\big\}$ and $\big\{\|\bm \Lambda_1^k-\bm\Lambda_1^{k-1}\|_F^2\big\}$ in (\ref{lag_def_eq}) are also bounded. Therefore, from (\ref{lag_def_ineq}), (\ref{lag_def_eq}) and (\ref{sum_mu}) follows that $\big\{\mathcal{L}_{\mu_1^k,\mu_2^k}\big(\bm J^{k+1},\bm C_1^{k+1},\bm C_2^{k+1}, \bm\Lambda_1^{k},\bm\Lambda_2^{k}\big)\big\}$ is upper-bounded.

Furthermore, it holds that:
\begin{equation}
\begin{split}
&\frac{1}{2}\big\|\bm X-\bm X \bm J^{k} \big\|^2_F+\lambda \psi_{\bm B}({\sigma}(\bm C_1^k))+\tau \psi_{\bm B}(\bm C_2^k)\\&=\mathcal{L}_{\mu_1^{k-1},\mu_2^{k-1}}\big(\bm J^{k},\bm C_1^{k},\bm C_2^k,\bm\Lambda_1^{k-1},\bm\Lambda_2^{k-1}\big)-\frac{1}{2}\frac{\big\|\bm \Lambda_1^k\big\|_F^2-\big\|\bm \Lambda_1^{k-1}\big\|_F^2}{\mu_1^{k-1}}\\&-\frac{1}{2}\frac{\big\|\bm \Lambda_2^k\big\|_F^2-\big\|\bm \Lambda_2^{k-1}\big\|_F^2}{\mu_2^{k-1}} 
\end{split}
\end{equation}
Since $\big\{\mathcal{L}_{\mu_1^{k-1},\mu_2^{k-1}}\big(\bm J^{k},\bm C_1^{k},\bm C_2^k, \bm\Lambda_1^{k-1},\bm\Lambda_2^{k-1}\big)\big\}$ is upper-bounded and $\big\{\bm \Lambda_1^{k}\big\}$ and $\big\{\bm \Lambda_2^{k}\big\}$ are both bounded, $\big\{\bm J^{k}\big\}$ is also bounded.

The boundedness of $\big\{\bm C_1^{k}\big\}$ and $\big\{\bm C_2^{k}\big\}$ follows from the update rules for Lagrange multipliers in $(\ref{Lambda_rule})$ and boundedness of $\big\{\bm J^{k}\big\}$, $\big\{\bm \Lambda_1^{k}\big\}$, and $\big\{\bm \Lambda_2^{k}\big\}$. 

Therefore, sequence $\big\{\big(\bm J^k, \bm C_1^k, \bm C_2^k, \bm \Lambda_1^k, \bm \Lambda_2^k\big)\big\}$ generated by Algorithm 1 is bounded. Bolzano-Weierstrass theorem \cite{Bartle11} then guarantees the existence of a convergent subsequence.\\

\textit{Proof of Theorem \ref{thm1}}:\\

Let $\big(\bm J^*, \bm C_1^*, \bm C_2^*, \bm \Lambda_1^*, \bm \Lambda_2^*\big)$ be a critical point of (\ref{LRSSC_gmc_reform}). The KKT conditions are derived as follows:
\begin{equation} \label{KKT_cond}
\begin{split}
&(1)\quad \bm J^*-\bm C_1^*=\bm{0},\\ 
&(2)\quad \bm J^*-\bm C_2^*=\bm{0},\\ &(3)-\bm{X}^T\big(\bm{X}-\bm{XJ^*}\big)+\bm{\Lambda}_1^*+\bm{\Lambda}_2^*=\bm{0}\\
&(4)\quad \bm{\Lambda}_1^* \in \partial_{_{\bm C_1}}\lambda \psi_{\bm B}\big(\sigma(\bm C_1^*)\big)\\ 
&(5) \quad
\bm{\Lambda}_2^* \in \partial_{_{\bm C_2}}\tau \psi_{\bm B}\big(\bm C_2^*\big).
\end{split}
\end{equation}
From the 1st and 4th KKT conditions, it follows:
\begin{equation} \label{4th_KKT_eq}
\begin{split}
\bm J^*+\frac{\bm \Lambda_1^*}{\mu_1^*}&\in \bm J^*+\frac{\lambda}{\mu_1^*}\partial_{_{\bm C_1}}\psi_{\bm B}\big(\sigma(\bm C_1^*)\big)\\
&= \bm C_1^*+\frac{\lambda}{\mu_1^*}
\partial_{_{\bm C_1}}\psi_{\bm B}\big(\sigma(\bm C_1^*)\big).
\end{split}
\end{equation}
Let $\bm{U}_1\bm \Sigma_1\bm{V}_1^T$ be the SVD of matrix $\bm C_1^*$. Using Proposition \ref{propos1}, the right hand side of (\ref{4th_KKT_eq}) equals:
\begin{equation}
\begin{split}
\bm C_1^*+\frac{\lambda}{\mu_1^*} \partial_{\bm C_1}\psi_{\bm B}\big(\sigma(\bm C_1^*)\big)&=\bm{U}_1\bm \Sigma_1\bm{V}_1^T+\frac{\lambda}{\mu_1^*}\bm{U}_1\partial_{\bm \Sigma_1}\big(\psi_{\bm B}\big(\bm \Sigma_1\big)\big)\bm{V}_1^T\\
&=\bm U_1\big(\bm{\Sigma_1}+\frac{\lambda}
{\mu_1^*} \partial_{\bm{\Sigma_1}}\big(\psi_{\bm B}\big(\bm \Sigma_1\big)\big)\bm{V}_1^T\\
&\triangleq \bm{U}_1Q_{a_1,b_1}\big(\bm \Sigma_1\big)\bm{V}_1^T,
\end{split}
\end{equation}
where $a_1=\mu_1^*/\lambda$ and $b_1=\sqrt{\mu_1^*\gamma/\lambda}$, $0< \gamma \leq 1$. The scalar function $Q_{a,b}$ is defined as $Q_{a,b}(x)\triangleq x+ \frac{1}{a}\partial\phi_{b}(x)$, where $\phi_{b}$ is the scaled MC penalty defined in (\ref{scaled_MC_penalty}). $Q_{a_1,b_1}(x)$ is applied element-wise to singular values of matrix $\bm C_1^*$.

Let $\bm U_2 \bm \Sigma_2\bm V_2^T$ be the SVD of matrix $\big(\bm J^*+\bm{\Lambda}_1^*/\mu_1^* \big)$. From (\ref{4th_KKT_eq}), we get the following relation:
\begin{equation} \label{C1_deriv_firm_step2}
\bm U_2\bm \Sigma_2\bm V_2^T\in\bm{U}_1Q_{a_1,b_1}\big(\bm \Sigma_1\big)\bm{V}_1^T.
\end{equation}
It is easy to verify that $Q_{a_1,b_1}$ is a monotone function \cite{KimLee16, ShenWen14} and
 $Q_{a_1,b_1}^{-1}(x)=\Theta\big(x; \frac{\lambda}{\mu_1^*}, \frac{\lambda}{\gamma\mu_1^*}\big)$ for $0<\gamma \leq 1$ , where $\Theta$ is the firm thresholding function defined in (\ref{firm_thr}). Applying $Q_{a_1,b_1}^{-1}$ to both sides of (\ref{C1_deriv_firm_step2}) and replacing $\bm{U}_1\bm \Sigma_1\bm{V}_1^T=\bm C_1^*$, we get: 
\begin{equation} \label{C1_KKT_transf}
\bm C_1^*=\bm U_2 Q_{a_1, b_1}^{-1}\big(\bm \Sigma_2\big)\bm V_2^T=\bm U_2\Theta\Big( \bm \Sigma_2; \frac{\lambda}{\mu_1^*}, \frac{\lambda}{\gamma\mu_1^*}\Big)\bm V_2^T.
\end{equation}
Similarly, from the 2nd and 5th KKT conditions, we have the following relations:
\begin{equation} \label{C2_deriv_firm}
\begin{split}
\bm J^*+\frac{\bm{\Lambda}_2^*}{\mu_2^*}&\in \bm J^*+\frac{\tau}{\mu_2^*}
\partial_{_{\bm C_2}}\psi_{\bm B}\big(\bm C_2^*\big)\\&=\bm C_2^*+\frac{\tau}{\mu_2^*}
\partial_{_{\bm C_2}}\psi_{\bm B}\big(\bm C_2^*\big)\triangleq Q_{a_2,b_2}\big(\bm C_2^*),
\end{split}
\end{equation}
where $a_2=\mu_2^*/\tau$ and $b_2=\sqrt{\mu_2^*\gamma/\tau}$, $0< \gamma \leq 1$.\\
Again, applying $Q_{a_2,b_2}^{-1}$ to both sides of (\ref{C2_deriv_firm}), we get the following equations:
\begin{equation} \label{C2_KKT_transf}
\bm C_2^*=Q_{a_2,b_2}^{-1}\Big(\bm J^*+\frac{\bm{\Lambda}_2^*}{\mu_2^*}\Big)= \Theta\Big(\bm J^*+\frac{\bm\Lambda_2^*}{\mu_2^*}; 
\frac{\tau}{\mu_2^*},\frac{\tau}{\gamma\mu_2^*}\Big).
\end{equation}
Therefore, the 4th and 5th KKT conditions can be rewritten as:
\begin{equation} \label{4_5th_KKT_rewritten}
\begin{split}
&(4)\quad \bm C_1^*=\bm U_2\Theta\Big( \bm \Sigma_2; \frac{\lambda}{\mu_1^*}, \frac{\lambda}{\gamma\mu_1^*}\Big)\bm V_2^T\\ 
&(5)\quad \bm C_2^*=\Theta\Big(\bm J^*+\frac{\bm\Lambda_2^*}{\mu_2^*}; \frac{\tau}{\mu_2^*}, \frac{\tau}{\gamma\mu_2^*}\Big).
\end{split}
\end{equation}

We now show that KKT conditions are satisfied when assumptions of Theorem \ref{thm1} hold. From (\ref{Lambda_rule}) we have:
\begin{equation}\label{KKT_lambda}
\begin{split}
{\bm\Lambda_1^{k+1}}-{\bm\Lambda_1^k}&=\mu_1^k\big(\bm J^{k+1}-\bm C_1^{k+1}\big)\\
{\bm\Lambda_2^{k+1}}-{\bm\Lambda_2^k}&=\mu_2^k\big(\bm J^{k+1}-\bm C_2^{k+1}\big).\\
\end{split}
\end{equation}
Since by the assumption $\big(\bm{\Lambda}_1^{k+1}-\bm{\Lambda}_1^{k}\big)\rightarrow \bm{0}$ and $\big(\bm{\Lambda}_2^{k+1}-\bm{\Lambda}_2^{k}\big)\rightarrow\bm{0}$, then the first two KKT conditions are satisfied.

From the update rule in (\ref{J_rule}), we have:
\begin{equation} \label{KKT_J}
\begin{split}
\big[\bm{X}^T\bm{X}&+(\mu_1^{k}+\mu_2^{k})\bm{I}\big]\big(\bm J^{k+1}-\bm J^{k}\big)=\mu_1^{k}\big({\bm C_1}^{k}-\bm J^{k}\big)\\&+\mu_2^{k}\big({\bm C_2}^{k}-\bm J^{k}\big)-{\bm{\Lambda}_1}^{k}-{\bm{\Lambda}_2}^{k}+ \bm X^T\big(\bm X- \bm X \bm J^{k}\big).
\end{split}
\end{equation}
From the the first two conditions, it follows that when $\bm J^{k+1}-\bm J^{k}\rightarrow \bm{0}$, the 3rd KKT condition is satisfied.

Next, using the update for $\bm C_1$ in (\ref{C1_update_gmc}) we obtain the following equation:
\begin{equation} \label{KKT_C1}
\bm C_1^{k+1}-\bm C_1^{k}=\bm U_2\Theta\Big(\bm \Sigma_2; \frac{\lambda}{\mu_1^k}, \frac{\lambda}{\gamma\mu_1^k}\Big){\bm V_2}^T-\bm C_1^{k},
\end{equation}
where $\bm U_2\bm \Sigma_2 {\bm V_2}^T$ is the SVD of matrix $\big(\bm J^{k+1}+{\bm{\Lambda}_1}^{k}/\mu_1^{k}\big)$.\\
From the update rule for $\bm C_2$ in (\ref{C2_update_gmc}) it follows:
\begin{equation} \label{KKT_C2}
\bm C_2^{k+1}-\bm C_2^{k}=\Theta\Big(\bm J^{k+1}+\frac{\bm\Lambda_2^k}{\mu_2^k}; \frac{\tau}{\mu_2^k}, \frac{\tau}{\gamma\mu_2^k}\Big)-\bm C_2^{k}.
\end{equation}
When $\bm C_1^{k+1}-\bm C_1^{k}\rightarrow \bm{0}$ and $\bm C_2^{k+1}-\bm C_2^{k}\rightarrow \bm{0}$, follow the equations in (\ref{4_5th_KKT_rewritten}).
Since $\big\{ Y^k\big\}_{k=1}^\infty$ is bounded and equations (\ref{KKT_lambda}), (\ref{KKT_J}), (\ref{KKT_C1}), (\ref{KKT_C2}) go to zero, we conclude that the sequence $\big\{Y^{k}\big\}_{k=1}^\infty$ asymptotically satisfies the KKT conditions in (\ref{KKT_cond}).

\section*{Acknowledgment}

The authors would like to thank Shin Matsushima for his helpful comments on improving the statements of Proposition 2 and Theorem 1. The authors would also like to thank the Editor and reviewers for their comments which helped to improve this paper.

\ifCLASSOPTIONcaptionsoff
  \newpage
\fi



\bibliographystyle{IEEEtran}
\bibliography{IEEEabrv,library}

\begin{thebibliography}{10}
\providecommand{\url}[1]{#1}
\csname url@samestyle\endcsname
\providecommand{\newblock}{\relax}
\providecommand{\bibinfo}[2]{#2}
\providecommand{\BIBentrySTDinterwordspacing}{\spaceskip=0pt\relax}
\providecommand{\BIBentryALTinterwordstretchfactor}{4}
\providecommand{\BIBentryALTinterwordspacing}{\spaceskip=\fontdimen2\font plus
\BIBentryALTinterwordstretchfactor\fontdimen3\font minus
  \fontdimen4\font\relax}
\providecommand{\BIBforeignlanguage}[2]{{%
\expandafter\ifx\csname l@#1\endcsname\relax
\typeout{** WARNING: IEEEtran.bst: No hyphenation pattern has been}%
\typeout{** loaded for the language `#1'. Using the pattern for}%
\typeout{** the default language instead.}%
\else
\language=\csname l@#1\endcsname
\fi
#2}}
\providecommand{\BIBdecl}{\relax}
\BIBdecl

\bibitem{Wang04}
X.~Wang and X.~Tang, ``A unified framework for subspace face recognition,''
  \emph{IEEE Trans. Pattern Anal. Mach. Intell.}, vol.~26, no.~9, pp.
  1222--1228, 2004.

\bibitem{Wright09}
J.~Wright, A.~Y. Yang, A.~Ganesh, S.~S. Sastry, and Y.~Ma, ``Robust face
  recognition via sparse representation,'' \emph{IEEE Trans. Pattern Anal.
  Mach. Intell.}, vol.~31, no.~2, pp. 210--227, 2009.

\bibitem{Rao10}
S.~Rao, R.~Tron, R.~Vidal, and Y.~Ma, ``Motion segmentation in the presence of
  outlying, incomplete, or corrupted trajectories,'' \emph{IEEE Trans. Pattern
  Anal. Mach. Intell.}, vol.~32, no.~10, pp. 1832--1845, 2010.

\bibitem{Morsier16}
F.~de~Morsier, M.~Borgeaud, V.~Gass, J.~P. Thiran, and D.~Tuia, ``Kernel
  low-rank and sparse graph for unsupervised and semi-supervised classification
  of hyperspectral images,'' \emph{IEEE Trans. Geosci. Remote Sens.}, vol.~54,
  no.~6, pp. 3410--3420, 2016.

\bibitem{Song16}
H.~Song, W.~Yang, N.~Zhong, and X.~Xu, ``Unsupervised classification of polsar
  imagery via kernel sparse subspace clustering,'' \emph{IEEE Geosci. Remote
  Sens. Lett.}, vol.~13, no.~10, pp. 1487--1491, 2016.

\bibitem{Liew11}
A.~W.-C. Liew, N.-F. Law, and H.~Yan, ``Missing value imputation for gene
  expression data: computational techniques to recover missing data from
  available information,'' \emph{Brief. Bioinf.}, vol.~12, no.~5, pp. 498--513,
  2011.

\bibitem{Vidal11}
R.~Vidal, ``Subspace clustering,'' \emph{\textit{IEEE Signal Process. Mag.}},
  vol.~28, no.~2, pp. 52--68, 2011.

\bibitem{Shi2000}
J.~Shi and J.~Malik, ``Normalized cuts and image segmentation,''
  \emph{\textit{IEEE Trans. Pattern Anal. Mach. Intell.}}, vol.~22, no.~8, pp.
  888--905, 2000.

\bibitem{Ng01}
A.~Y. Ng, M.~I. Jordan, and Y.~Weiss, ``On spectral clustering: Analysis and an
  algorithm,'' in \emph{\textit{Advan. Neural Inf. Process. Syst. 14}}, 2001,
  pp. 849--856.

\bibitem{vonLuxburg2007}
U.~von Luxburg, ``A tutorial on spectral clustering,'' \emph{\textit{Stat.
  Comput.}}, vol.~17, no.~4, pp. 395--416, 2007.

\bibitem{Elhamifar13}
E.~Elhamifar and R.~Vidal, ``Sparse subspace clustering: Algorithm, theory, and
  applications,'' \emph{\textit{IEEE Trans. Pattern Anal. Mach. Intell.}},
  vol.~35, no.~11, pp. 2765--2781, 2013.

\bibitem{Peng17}
X.~Peng, Z.~Yu, Z.~Yi, and H.~Tang, ``Constructing the {L}2-graph for robust
  subspace learning and subspace clustering,'' \emph{IEEE Trans. Cybern.},
  vol.~47, no.~4, pp. 1053--1066, 2017.

\bibitem{Liu10b}
G.~Liu and Y.~Lin, Z. annd~Yu, ``Robust subspace segmentation by low-rank
  representation,'' in \emph{\textit{26th Int. Conf. Mach. Learn.}}, 2010, pp.
  663--670.

\bibitem{Liu13}
G.~Liu, Z.~Lin, S.~Yan, J.~Sun, Y.~Yu, and Y.~Ma, ``Robust recovery of subspace
  structures by low-rank representation,'' \emph{\textit{IEEE Trans. Pattern
  Anal. Mach. Intell.}}, vol.~35, no.~1, pp. 171--184, 2013.

\bibitem{WangWang17}
J.~Wang, X.~Wang, F.~Tian, C.~H. Liu, and H.~Yu, ``Constrained low-rank
  representation for robust subspace clustering,'' \emph{IEEE Trans. Cybern.},
  vol.~47, no.~12, pp. 4534--4546, 2017.

\bibitem{Li16}
C.~G. Li and R.~Vidal, ``A structured sparse plus structured low-rank framework
  for subspace clustering and completion,'' \emph{IEEE Trans. Signal Process.},
  vol.~64, no.~24, pp. 6557--6570, 2016.

\bibitem{Favaro11}
P.~Favaro, R.~Vidal, and A.~Ravichandran, ``A closed form solution to robust
  subspace estimation and clustering,'' in \emph{\textit{IEEE Conf. Comput.
  Vision Pattern Recognit.}}, 2011, pp. 1801--1807.

\bibitem{Candes09}
E.~J. Cand{\`e}s and B.~Recht, ``Exact matrix completion via convex
  optimization,'' \emph{Found. Comput. Math.}, vol.~9, no.~6, p. 717, 2009.

\bibitem{Candes10}
E.~J. Cand{\`e}s and T.~Tao, ``The power of convex relaxation: Near-optimal
  matrix completion,'' \emph{IEEE Trans. Inf. Th.}, vol.~56, no.~5, pp.
  2053--2080, 2010.

\bibitem{Candes11}
E.~J. Cand{\`e}s, X.~Li, Y.~Ma, and J.~Wright, ``Robust principal component
  analysis?'' \emph{J. ACM (JACM)}, vol.~58, no.~3, p.~11, 2011.

\bibitem{Elhamifar09}
E.~Elhamifar and R.~Vidal, ``Sparse subspace clustering,'' in
  \emph{\textit{IEEE Conf. Comput. Vision Pattern Recognit.}}, 2009, pp.
  2790--2797.

\bibitem{Candes05}
E.~J. Cand{\`e}s and T.~Tao, ``Decoding by linear programming,'' \emph{IEEE
  Trans. Inf. Th.}, vol.~51, no.~12, pp. 4203--4215, 2005.

\bibitem{Donoho06}
D.~L. Donoho, ``For most large underdetermined systems of linear equations the
  minimal $\ell_1$-norm solution is also the sparsest solution,'' \emph{Commun.
  Pure Applied Math.}, vol.~59, no.~6, pp. 797--829, 2006.

\bibitem{Tang14}
K.~Tang, R.~Liu, Z.~Su, and J.~Zhang, ``Structure-constrained low-rank
  representation,'' \emph{\textit{IEEE Trans. Neural Netw. Learn. Syst.}},
  vol.~25, no.~12, pp. 2167--2179, 2014.

\bibitem{Tang16}
K.~Tang, D.~B. Dunsonc, Z.~Sub, R.~Liub, J.~Zhanga, and J.~Dong, ``Subspace
  segmentation by dense block and sparse representation,'' \emph{\textit{Neural
  Netw.}}, vol.~75, no.~1, pp. 66--76, 2016.

\bibitem{Elhamifar10}
E.~Elhamifar and R.~Vidal, ``Clustering disjoint subspaces via sparse
  representation,'' in \emph{\textit{IEEE Int. Conf. Acoustics, Speech Signal
  Process.}}, 2010, pp. 1926--1929.

\bibitem{Nasihatkon11}
B.~Nasihatkon and R.~Hartley, ``Graph connectivity in sparse subspace
  clustering,'' in \emph{{IEEE} Conf. Comput. Vision Pattern Recognit.}, 2011,
  pp. 2137--2144.

\bibitem{Zhuang12}
L.~Zhuang, H.~Gao, Z.~Lin, Y.~Ma, X.~Zhang, and N.~Yu, ``Non-negative low rank
  and sparse graph for semi-supervised learning,'' in \emph{{IEEE} Conf.
  Comput. Vision Pattern Recognit.}, 2012, pp. 2328--2335.

\bibitem{Wang13}
Y.-X. Wang, H.~Xu, and C.~Leng, ``Provable subspace clustering: {W}hen {LRR}
  meets {SSC},'' in \emph{\textit{Advan. Neural Inf. Process. Syst. 26}}, 2013,
  pp. 64--72.

\bibitem{Brbic18}
M.~Brbi{\'c} and I.~Kopriva, ``Multi-view low-rank sparse subspace
  clustering,'' \emph{Pattern Recognit.}, vol.~73, pp. 247--258, 2018.

\bibitem{Donoho03}
D.~L. Donoho and M.~Elad, ``Optimally sparse representation in general
  (nonorthogonal) dictionaries via $\ell_1$ minimization,'' \emph{Proc. Natl.
  Acad. Sci}, vol. 100, no.~5, pp. 2197--2202, 2003.

\bibitem{Larsson16}
V.~Larsson and C.~Olsson, ``Convex low rank approximation,'' \emph{Int. J.
  Comput. Vision}, vol. 120, no.~2, pp. 194--214, 2016.

\bibitem{Parekh16}
A.~Parekh and I.~W. Selesnick, ``Enhanced low-rank matrix approximation,''
  \emph{{IEEE} Signal Process. Lett.}, vol.~23, no.~4, pp. 493--497, 2016.

\bibitem{Selesnick17b}
I.~Selesnick, ``Sparse regularization via convex analysis,'' \emph{IEEE Trans.
  Signal Process.}, vol.~65, no.~17, pp. 4481--4494, 2017.

\bibitem{LuLin14}
C.~Lu, J.~Tang, S.~Yan, and Z.~Lin, ``Generalized nonconvex nonsmooth low-rank
  minimization,'' in \emph{{IEEE} Conf. Comput. Vision Pattern Recognit.},
  2014, pp. 4130--4137.

\bibitem{LuZhu15}
C.~Lu, C.~Zhu, C.~Xu, S.~Yan, and Z.~Lin, ``Generalized singular value
  thresholding,'' in \emph{Proc. the 29th AAAI Conf. Artif. Intell.}, ser.
  AAAI'15.\hskip 1em plus 0.5em minus 0.4em\relax AAAI Press, 2015, pp.
  1805--1811.

\bibitem{Malek14}
M.~Malek-Mohammadi, M.~Babaie-Zadeh, and M.~Skoglund, ``Iterative concave rank
  approximation for recovering low-rank matrices,'' \emph{IEEE Trans. Signal
  Process.}, vol.~62, no.~20, pp. 5213--5226, 2014.

\bibitem{Kliesch16}
M.~Kliesch, R.~Kueng, J.~Eisert, and D.~Gross, ``Improving compressed sensing
  with the diamond norm,'' \emph{IEEE Trans. Inf. Th.}, vol.~62, no.~12, pp.
  7445--7463, 2016.

\bibitem{LuTang16}
C.~Lu, J.~Tang, S.~Yan, and Z.~Lin, ``Nonconvex nonsmooth low rank minimization
  via iteratively reweighted nuclear norm,'' \emph{IEEE Trans. Image Process.},
  vol.~25, no.~2, pp. 829--839, 2016.

\bibitem{Yuan17}
X.~T. Yuan and Q.~Liu, ``Newton-type greedy selection methods for $\ell_0$
  -constrained minimization,'' \emph{IEEE Trans. Pattern Anal. Mach. Intell.},
  vol.~39, no.~12, pp. 2437--2450, 2017.

\bibitem{Peharz12}
R.~Peharz and F.~Pernkopf, ``Sparse nonnegative matrix factorization with
  $\ell^0$-constraints,'' \emph{Neurocomput.}, vol.~80, pp. 38--46, 2012.

\bibitem{Blumensath09}
T.~Blumensath and M.~E. Davies, ``Iterative hard thresholding for compressed
  sensing,'' \emph{Applied Comput. Harm. Anal.}, vol.~27, no.~3, pp. 265--274,
  2009.

\bibitem{Mohimani09}
H.~Mohimani, M.~Babaie-Zadeh, and C.~Jutten, ``A fast approach for overcomplete
  sparse decomposition based on smoothed $\ell ^{0}$ norm,'' \emph{IEEE Trans.
  Signal Process.}, vol.~57, no.~1, pp. 289--301, 2009.

\bibitem{Chartrand07}
R.~Chartrand, ``Exact reconstruction of sparse signals via nonconvex
  minimization,'' \emph{IEEE Signal Process. Lett.}, vol.~14, no.~10, pp.
  707--710, 2007.

\bibitem{Daubechies04}
I.~Daubechies, M.~Defrise, and C.~De~Mol, ``An iterative thresholding algorithm
  for linear inverse problems with a sparsity constraint,'' \emph{Commun. Pure
  Applied Math.}, vol.~57, no.~11, pp. 1413--1457, 2004.

\bibitem{XuChang12}
Z.~Xu, X.~Chang, F.~Xu, and H.~Zhang, ``$l_{1/2}$ regularization: A
  thresholding representation theory and a fast solver,'' \emph{IEEE Trans.
  Neural Netw. and Learn. Syst.}, vol.~23, no.~7, pp. 1013--1027, 2012.

\bibitem{Nikolova13}
M.~Nikolova, ``Description of the minimizers of least squares regularized with
  $\ell_0$-norm. {U}niqueness of the global minimizer,'' \emph{SIAM J. Imag.
  Sci.}, vol.~6, no.~2, pp. 904--937, 2013.

\bibitem{Yang16}
Y.~Yang, J.~Feng, N.~Jojic, J.~Yang, and T.~S. Huang, ``$\ell^{0}$-sparse
  subspace clustering,'' in \emph{Europ. Conf. on Comput. Vision}.\hskip 1em
  plus 0.5em minus 0.4em\relax Springer, 2016, pp. 731--747.

\bibitem{Jiang16}
W.~Jiang, J.~Liu, H.~Qi, and Q.~Dai, ``Robust subspace segmentation via
  nonconvex low rank representation,'' \emph{Inf. Sci.}, vol. 340, no.~C, pp.
  144--158, 2016.

\bibitem{Zhang16}
X.~Zhang, C.~Xu, X.~Sun, and G.~Baciu, ``Schatten-q regularizer constrained low
  rank subspace clustering model,'' \emph{Neurocomp.}, vol. 182, no.~C, pp.
  36--47, 2016.

\bibitem{Zhang18}
H.~Zhang, J.~Yang, F.~Shang, C.~Gong, and Z.~Zhang, ``{LRR} for subspace
  segmentation via tractable {S}chatten-$p$ norm minimization and
  factorization,'' \emph{IEEE Trans. Cyber.}, doi:10.1109/TCYB.2018.2811764.

\bibitem{Cheng17}
W.~Cheng, M.~Zhao, N.~Xiong, and K.~T. Chui, ``Non-convex sparse and low-rank
  based robust subspace segmentation for data mining,'' \emph{Sensors},
  vol.~17, no.~7, p. 1633, 2017.

\bibitem{Doan16}
X.~V. Doan and S.~Vavasis, ``Finding the largest low-rank clusters with ky fan
  2-k-norm and $\ell_1$-norm,'' \emph{SIAM J. Optim.}, vol.~26, no.~1, pp.
  274--312, 2016.

\bibitem{Zhang19}
X.~Z.~H. Sun, Z.~Liu, Z.~Ren, Q.~Cui, and Y.~Li, ``Robust low-rank kernel
  multi-view subspace clustering based on the {S}chatten p-norm and
  correntropy,'' \emph{Inf. Sciences}, vol. 477, pp. 430 -- 447, 2019.

\bibitem{Zuo13}
W.~Zuo, D.~Meng, L.~Zhang, X.~Feng, and D.~Zhang, ``A generalized iterated
  shrinkage algorithm for non-convex sparse coding,'' in \emph{{IEEE} Int.
  Conf. Comput. Vision}, 2013, pp. 217--224.

\bibitem{Zheng17}
L.~Zheng, A.~Maleki, H.~Weng, X.~Wang, and T.~Long, ``Does $\ell
  _{p}$-minimization outperform $\ell_{1}$-minimization?'' \emph{IEEE Trans.
  Inf. Th.}, vol.~63, no.~11, pp. 6896--6935, Nov 2017.

\bibitem{Soubies15}
E.~Soubies, L.~Blanc-F{\'e}raud, and G.~Aubert, ``{A continuous exact $\ell_0$
  penalty (CEL0) for least squares regularized problem},'' \emph{{SIAM J. Imag.
  Sci. (SJIMS)}}, vol.~8, no.~3, pp. pp. 1607--1639 (33 pages), 2015.

\bibitem{Gao97}
H.-Y. Gao and A.~G. Bruce, ``Waveshrink with firm shrinkage,'' \emph{Stat.
  Sinica}, pp. 855--874, 1997.

\bibitem{Blumensath08}
T.~Blumensath and M.~E. Davies, ``Iterative thresholding for sparse
  approximations,'' \emph{J. Fourier Analy. Appl.}, vol.~14, no.~5, pp.
  629--654, 2008.

\bibitem{Le13}
H.~Y. Le, ``Generalized subdifferentials of the rank function,'' \emph{Optim.
  Lett.}, vol.~7, no.~4, pp. 731--743, 2013.

\bibitem{Liang16}
J.~Liang, J.~Fadili, and G.~Peyr{\'e}, ``A multi-step inertial forward-backward
  splitting method for non-convex optimization,'' in \emph{Advan. Neural Inf.
  Process. Syst. 29}, 2016, pp. 4035--4043.

\bibitem{YuYao13}
Y.-L. Yu, ``Better approximation and faster algorithm using the proximal
  average,'' in \emph{Advan. Neural Inf. Process. Syst.}, 2013, pp. 458--466.

\bibitem{YuXun15}
Y.~Yu, X.~Zheng, M.~Marchetti-Bowick, and E.~Xing, ``Minimizing nonconvex
  non-separable functions,'' in \emph{Proc. 18th Int. Conf. Artif. Intell.
  Stat.}, ser. Proc. Mach. Learn. Research, vol.~38, 2015, pp. 1107--1115.

\bibitem{LinWei16}
X.~Lin and G.~Wei, ``Generalized non-convex non-smooth sparse and low rank
  minimization using proximal average,'' \emph{Neurocomput.}, vol. 174, no.
  Part B, pp. 1116 -- 1124, 2016.

\bibitem{Boyd11}
S.~Boyd, N.~Parikh, E.~Chu, B.~Peleato, and J.~Eckstein, ``Distributed
  optimization and statistical learning via the alternating direction method of
  multipliers,'' \emph{Found. Trends Mach. Learn.}, vol.~3, no.~1, pp. 1--122,
  2011.

\bibitem{Sun14}
D.~L. Sun and C.~Fevotte, ``Alternating direction method of multipliers for
  non-negative matrix factorization with the beta-divergence,'' in \emph{{IEEE}
  Int. Conf. Acoustics, Speech Signal Process.}, 2014, pp. 6201--6205.

\bibitem{Zhang14}
R.~Zhang and J.~T. Kwok, ``Asynchronous distributed {ADMM} for consensus
  optimization,'' in \emph{Int. Conf. Mach. Learn.}, 2014, pp. 1701--1709.

\bibitem{Dong13}
B.~Dong and Y.~Zhang, ``An efficient algorithm for $\ell _{0}$ minimization in
  wavelet frame based image restoration,'' \emph{J. Sci. Comp.}, vol.~54, no.
  2-3, pp. 350--368, 2013.

\bibitem{Hong16}
M.~Hong, Z.-Q. Luo, and M.~Razaviyayn, ``Convergence analysis of alternating
  direction method of multipliers for a family of nonconvex problems,''
  \emph{SIAM J. Optim.}, vol.~26, no.~1, pp. 337--364, 2016.

\bibitem{WangYin15}
Y.~Wang, W.~Yin, and J.~Zeng, ``Global convergence of {ADMM} in nonconvex
  nonsmooth optimization,'' \emph{J. Sci. Comput}, 2018, {A}vailable:
  https://doi.org/10.1007/s10915-018-0757-z.

\bibitem{LiTing15}
G.~Li and T.~K. Pong, ``Global convergence of splitting methods for nonconvex
  composite optimization,'' \emph{SIAM J. Optim.}, vol.~25, no.~4, pp.
  2434--2460, 2015.

\bibitem{WangCao15}
F.~Wang, W.~Cao, and Z.~Xu, ``Convergence of multi-block {B}regman {ADMM} for
  nonconvex composite problems,'' \emph{Sci. China Inf. Sci.}, vol.~61, no.
  122101, 2018, {A}vailable: https://doi.org/10.1007/s11432-017-9367-6.

\bibitem{Attouch13}
H.~Attouch, J.~Bolte, and B.~F. Svaiter, ``Convergence of descent methods for
  semi-algebraic and tame problems: proximal algorithms, forward--backward
  splitting, and regularized {G}auss-{S}eidel methods,'' \emph{Math.
  Programm.}, vol. 137, no.~1, pp. 91--129, 2013.

\bibitem{Vidal14}
R.~Vidal and P.~Favaro, ``Low rank subspace clustering ({LRSC}),''
  \emph{Pattern Recognit. Lett.}, vol.~43, pp. 47--61, 2014.

\bibitem{Lewis05}
A.~S. Lewis and H.~S. Sendov, ``Nonsmooth analysis of singular values. part
  {I}: Theory,'' \emph{Set-Valued Anal.}, vol.~13, no.~3, pp. 213--241, 2005.

\bibitem{Sun17}
T.~Sun and L.~Cheng, ``Convergence of iterative hard-thresholding algorithm
  with continuation,'' \emph{Optim. Lett.}, vol.~11, no.~4, pp. 801--815, 2017.

\bibitem{Zhang10}
C.-H. Zhang, ``Nearly unbiased variable selection under minimax concave
  penalty,'' \emph{The Annals Stat.}, vol.~38, no.~2, pp. 894--942, 2010.

\bibitem{Kuhn51}
H.~W. Kuhn and A.~W. Tucker, ``Nonlinear programming,'' in \emph{Proc. 2nd
  Berkeley Symp. Math. Stat. Prob.}, 1951, pp. 481--492.

\bibitem{Cichocki06}
A.~Cichocki, R.~Zdunek, and S.~Amari, ``New algorithms for non-negative matrix
  factorization in applications to blind source separation,'' in \emph{{IEEE}
  Int. Conf. Acoustics Speech Signal Process.}, vol.~5, 2006, pp. 621--624.

\bibitem{Lojasiewicz1993}
S.~Lojasiewicz, ``Sur la g\'eom\'etrie semi- et sous- analytique,''
  \emph{Annales de l'institut Fourier}, vol.~43, no.~5, pp. 1575--1595, 1993.

\bibitem{Dyer13}
E.~L. Dyer, A.~C. Sankaranarayanan, and R.~G. Baraniuk, ``Greedy feature
  selection for subspace clustering,'' \emph{J. Mach. Learn. Res.}, vol.~14,
  no.~1, pp. 2487--2517, 2013.

\bibitem{Heckel15}
R.~Heckel and H.~B{\"o}lcskei, ``Robust subspace clustering via thresholding,''
  \emph{{IEEE} Trans. Inf. Th.}, vol.~61, no.~11, pp. 6320--6342, 2015.

\bibitem{Park14}
D.~Park, C.~Caramanis, and S.~Sanghavi, ``Greedy subspace clustering,'' in
  \emph{Advan. Neural Inf. Process. Syst. 27}, 2014, pp. 2753--2761.

\bibitem{Yang18s}
Y.~Yang, J.~Feng, N.~Jojic, J.~Yang, and T.~S. Huang, ``Subspace learning by
  $\ell^{0}$--induced sparsity,'' \emph{Int. J. Comput. Vision}, vol. 126,
  no.~10, pp. 1138--1156, 2018.

\bibitem{Georghiades01}
A.~S. Georghiades, P.~N. Belhumeur, and D.~J. Kriegman, ``From few to many:
  Illumination cone models for face recognition under variable lighting and
  pose,'' \emph{{IEEE} Trans. Pattern Anal. Mach. Intell.}, vol.~23, no.~6, pp.
  643--660, 2001.

\bibitem{Lee05}
K.-C. Lee, J.~Ho, and D.~J. Kriegman, ``Acquiring linear subspaces for face
  recognition under variable lighting,'' \emph{{IEEE} Trans. Pattern Anal.
  Mach. Intell.}, vol.~27, no.~5, pp. 684--698, 2005.

\bibitem{Hastie98}
T.~Hastie and P.~Y. Simard, ``Metrics and models for handwritten character
  recognition,'' \emph{Stat. Sci.}, pp. 54--65, 1998.

\bibitem{Fanty91}
M.~Fanty and R.~Cole, ``Spoken letter recognition,'' in \emph{Advan. Neural
  Inf. Process. Syst.}, 1991, pp. 220--226.

\bibitem{Suter16}
F.~Chen and B.~W. Suter, ``Computing the proximity operator of the $\ell_{p}$
  norm with 0<p<1,'' \emph{{IET} Signal Process.}, vol.~10, pp. 557--565, 2016.

\bibitem{Shang18}
F.~Shang, J.~Cheng, Y.~Liu, Z.~Luo, and Z.~Lin, ``Bilinear factor matrix norm
  minimization for robust {PCA}: Algorithms and applications,'' \emph{IEEE
  Trans. Pattern Anal. Machine Intell.}, vol.~40, no.~9, pp. 2066--2080, 2018.

\bibitem{Bartle11}
R.~G. Bartle and D.~R. Sherbert, \emph{Introduction to real analysis},
  4th~ed.\hskip 1em plus 0.5em minus 0.4em\relax Wiley New York, 2011.

\bibitem{KimLee16}
E.~Kim, M.~Lee, and S.~Oh, ``Robust elastic-net subspace representation,''
  \emph{IEEE Trans. Image Process.}, vol.~25, no.~9, pp. 4245--4259, 2016.

\bibitem{ShenWen14}
Y.~Shen, Z.~Wen, and Y.~Zhang, ``Augmented lagrangian alternating direction
  method for matrix separation based on low-rank factorization,'' \emph{Optim.
  Meth. Softw.}, vol.~29, no.~2, pp. 239--263, 2014.

\end{thebibliography}
\end{document}